\documentclass{article}
\usepackage[dvipsnames]{xcolor}
\usepackage{scalerel}
\usepackage{multirow}
\usepackage[preprint]{corl_2026}
\usepackage[normalem]{ulem}

\usepackage{amsmath,amssymb,amsthm}
\usepackage{mathtools}
\usepackage{xspace}
\usepackage{booktabs}
\usepackage{graphicx}
\usepackage{algorithm}
\usepackage{algpseudocode}

\newtheorem{theorem}{Theorem}

\newtheorem{corollary}{Corollary}

\newcommand{\paper}{\textsc{VINE}\xspace}

\newcommand{\dataset}{\mathcal{D}}

\DeclareMathOperator{\E}{\mathbb{E}}

\usepackage{graphicx}
\usepackage{wrapfig}

\title{VINE: Taming Generative Control Policies for Reinforcement Learning}

\author{%
\textbf{Rushuai Yang}\textsuperscript{1,2} \quad
\textbf{Zhuo Han}\textsuperscript{1} \quad
\textbf{Houlin Li}\textsuperscript{1} \quad
\textbf{Hecheng Wang}\textsuperscript{1} \\
\textbf{Zhichao Wu}\textsuperscript{1} \quad
\textbf{Rui Zhang}\textsuperscript{1} \quad
\textbf{Zhaowei Zhang}\textsuperscript{3} \quad
\textbf{Zihong Chen}\textsuperscript{1} \\ 
\textbf{Xiaohan Yan}\textsuperscript{1} \quad
\textbf{Chiming Liu}\textsuperscript{1,\textdagger} \quad
\textbf{Yi Chen}\textsuperscript{2} \quad
\textbf{Wei Shan}\textsuperscript{1,\textdagger} \quad
\textbf{Maoqing Yao}\textsuperscript{1,\textdagger}
\\
\textsuperscript{1}AgiBot \quad
\textsuperscript{2}The Hong Kong University of Science and Technology \quad
\textsuperscript{3}Peking University \quad
\\
\textsuperscript{\textdagger}Corresponding Author
}

\begin{document}
\maketitle

\begin{abstract}
    Flow-matching policies have emerged as an effective policy parameterization for robot learning. They iteratively generate actions from noise, enabling highly expressive modeling of complex and multimodal action distributions. However, prior works observed that scaling these policies with value-gradient reinforcement learning (RL) often leads to training instability. Existing methods attribute this instability to iterative generation and therefore avoid end-to-end value-gradient optimization by sacrificing iterative generation, high expressiveness, or value-gradient optimization. Contrary to prior belief, we show the instability does not stem from iterative generation itself, but from the vanilla sampling strategy originally designed for behavior cloning, which becomes brittle under value-gradient RL. Motivated by this insight, we propose \paper, an RL-oriented sampling method that enables stable end-to-end value-gradient optimization for flow-matching policies. Instead of following a single flow trajectory, \paper reconstructs a new interpolation state at every denoising step, creating a stable differentiable path for value-gradient propagation while remaining compatible with the original flow-matching denoising process. As a result, \paper preserves the expressiveness and iterative generation of flow-matching without sacrificing end-to-end value-gradient optimization. Despite performing end-to-end backpropagation through all ten denoising steps, \paper achieves stable policy improvement and consistently outperforms state-of-the-art RL methods on the OGBench offline RL benchmark and real-world robotic manipulation task. Videos are available on our website: \href{https://agibottech.github.io/vine/}{https://agibottech.github.io/vine}.
\end{abstract}
\begin{figure}[h]
    \centering
    \includegraphics[width=\textwidth]{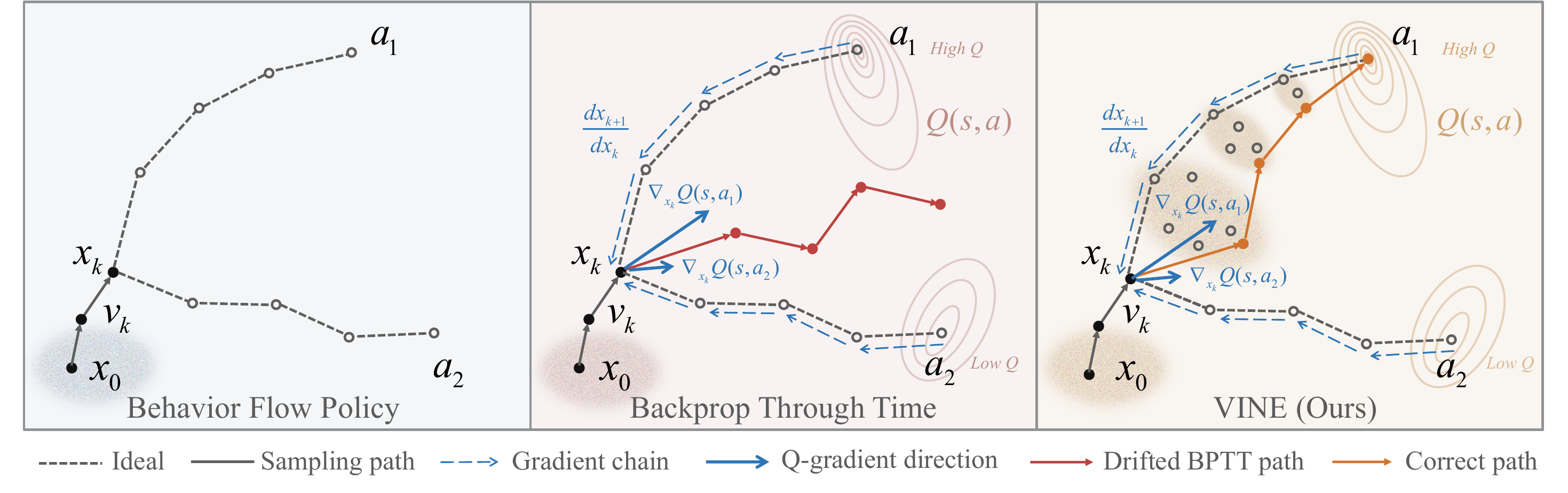}
\caption{\textbf{Left:} A flow policy trained by behavior cloning fits the multi-modal data distribution. \textbf{Middle:} Directly backpropagating the critic gradient through the denoising steps (value BPTT) destabilizes the trajectory. \textbf{Right:} VINE produces a stable denoising trajectory that supports value-gradient BPTT toward $a_1$.}
\end{figure}


\section{Introduction}
\label{sec:intro}
In robotic learning, generative control policies such as diffusion and flow-matching models have achieved remarkable success for behavior cloning in a wide range of manipulation and control tasks~\citep{zhao2023act,chi2023diffusion_policy,pi02024,intelligence2025pi}. In general, they start from randomly sampled noise, and iteratively refine the noise using a learned velocity or score field until an executable action is produced. By combining structured noise injection with iterative refinement, these policies can represent complex and multimodal action distributions with substantially greater expressive power than conventional policy parameterizations~\citep{chi2023diffusion_policy}. Because human demonstrations are not always optimal and do not directly align with reward maximization, prior work has begun to further improve these generative control policies with reinforcement learning~\citep{chen2025pi_,zhang2026reinflow,luo2025precise,yang2026aloe,yang2025dlr,yang2025beyond}. In this paradigm, the policy is trained to maximize expected returns, using a learned value function to estimate future returns as the optimization signal. Yet prior work has found that such approaches suffer from severe training instability when value-gradient optimization is applied to highly expressive generative policies~\citep{zhang2026sacflowsampleefficientreinforcement, qam2026, park2025flowqlearning,psenka2023learning,dhariwal2021diffusionmodelsbeatgans}.

To mitigate this problem, existing work generally attributes this instability to the iterative generation process, and therefore avoids propagating value gradients through the full denoising trajectory~\citep{xu2026reinforcement,liu2026value}. Specifically, prior methods typically adopt one of three strategies: (1) freezing the parameters of the generative control policy and using external guidance signals to steer the denoising dynamics~\citep{psenka2023learning,wagenmaker2025steering,dhariwal2021diffusionmodelsbeatgans,qam2026,fang2025diffusionactorcriticformulatingconstrained}; (2) using (residual) Gaussian policy or distilling multi-step iterative generation into a one-step policy~\citep{park2025flowqlearning,dong2025expo}; and (3) discarding the critic's action gradient entirely and using scalar value estimates as weighting signals only~\citep{peng2019advantage,hansen2023idql,dong2026valueflows}. While these approaches improve training stability, they all introduce fundamental compromises. Strategy (1) sacrifices expressiveness. It does not fully utilize representational capacity of generative model and the learnable component used for RL is restricted. Strategy (2) removes iterative generation, despite iterative refinement being widely recognized as a key advantage of generative policies. Strategy (3) sacrifices value-gradient optimization, which is generally regarded as more efficient \cite{park2025flowqlearning,xu2026rl}. In essence, existing methods often lead to suboptimal performance without addressing challenge arising from the combination of high expressiveness, iterative generation, and value-gradient optimization in reinforcement learning for generative control policies~\citep{dong2025expo,zhang2025gorl}.

In this paper, we revisit the source of training instability in reinforcement learning for generative control policies. We identify one important bottleneck as the noise sampling process, which is poorly aligned with the optimization dynamics of reinforcement learning. Existing noise sampling strategies are primarily designed for behavior cloning. They aim to reconstruct a fixed action distribution, with expressiveness mainly reflected in how well noise can be denoised into a static, often multimodal, data distribution~\citep{chi2023diffusion_policy}. While this design is highly effective for supervised distribution modeling, reinforcement learning requires more than reconstructing a fixed distribution. During RL fine-tuning, the policy is repeatedly optimized through value queries from a learned critic, whose gradients encourage the policy to shift probability mass toward higher-value actions~\citep{qam2026,zhang2026force}. This requires the iterative sampling process to transition smoothly among different high-value modes rather than simply reproduce the behavior distribution. A sampling process designed for static multimodal reconstruction is therefore forced to support a dynamic, value-driven mode-shifting process, creating the instability observed in value-gradient fine-tuning. This observation motivates a reinforcement-learning-oriented redesign of the sampling process. 

Our primary contribution is an RL-oriented sampling method, which we call \paper, for fine-tuning generative control policies with end-to-end value gradients while simultaneously preserving their expressive iterative sampling structure. \paper replaces the standard single-noise flow trajectory with a sequence of reconstructed interpolation states and intermediate noise injection, providing a stable differentiable path for value-gradient backpropagation through the sampler. Our analysis shows that this reconstruction keeps every network query compatible with the original flow-matching training semantics. As a result, \paper can be applied to pretrained flow-matching policies by modifying only the sampler implementation. Empirically, we evaluate \paper on challenging offline RL benchmarks and real-world control task, where it achieves state-of-the-art performance while improving the robustness of value-gradient fine-tuning.

\begin{figure*}[t]
    \centering
    \includegraphics[width=1\textwidth]{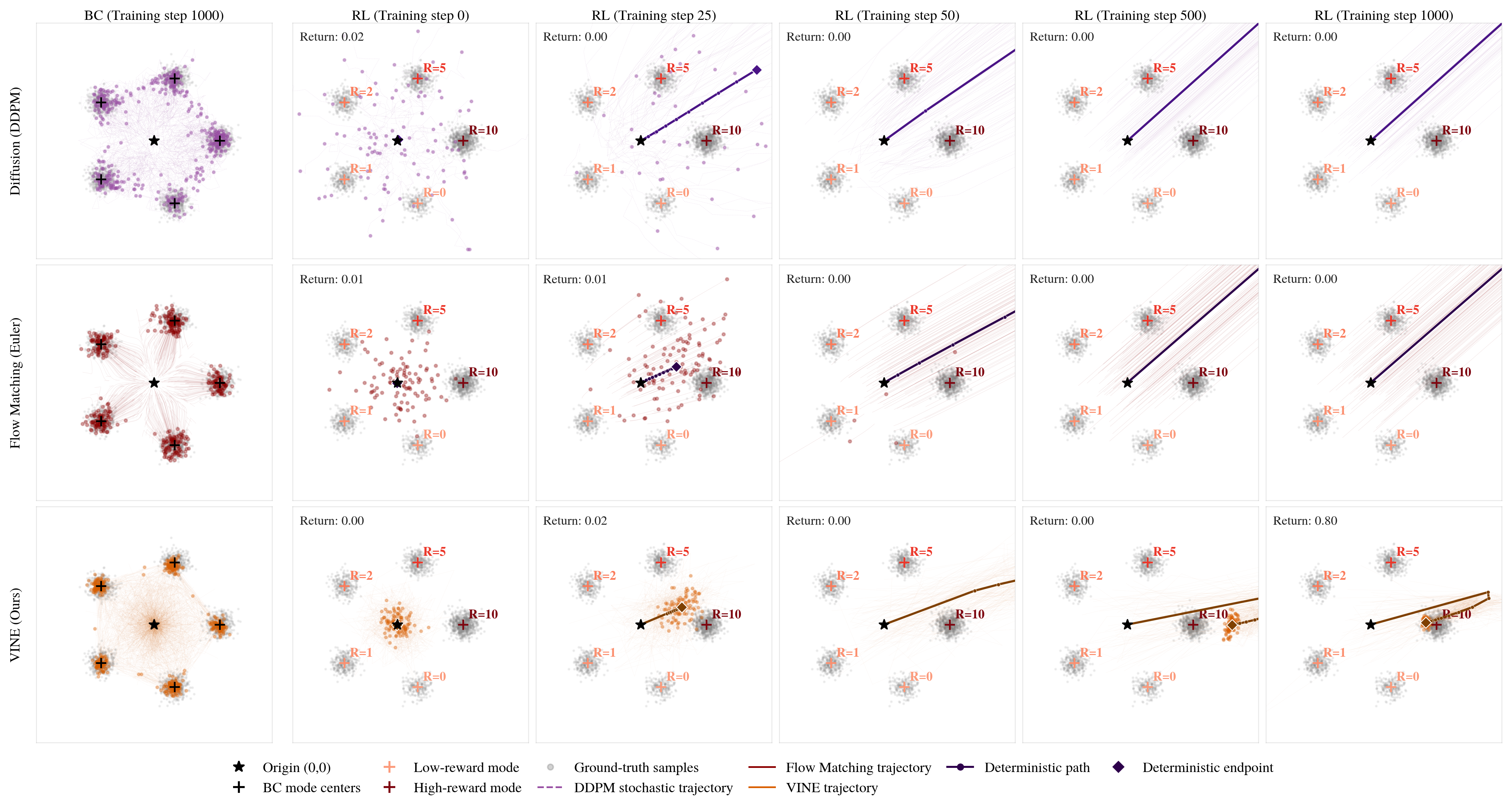}\caption{\textbf{Toy simulation of probability paths induced by different samplers under behavior cloning and offline value-gradient fine-tuning.}
\textbf{Left:} Under behavior cloning, all samplers recover the multimodal data distribution, but Euler trajectories form isolated petal-like paths, DDPM explores broadly with noisy endpoints, and \paper provides structured exploration and broader state coverage.
\textbf{Right:} After assigning different rewards ($R$) to the clusters and fine-tuning with value gradients, standard diffusion and Euler flow-matching samplers are prone to persistent directional errors, whereas \paper enables later refinement steps to correct early deviations and steer the trajectory toward high-value action regions.}
    \label{fig:toy_example}
\end{figure*}

\section{Preliminaries}
\label{sec:preliminaries}

\subsection{Reinforcement Learning and Actor-Critic Training}
\label{sec:prelim_rl}

We consider a Markov Decision Process $\mathcal{M} = (\mathcal{S}, \mathcal{A}, P, \gamma, R, \mu)$, where $\mathcal{S}$ is the state space, $\mathcal{A} = \mathbb{R}^d$ is the continuous action space, $P: \mathcal{S} \times \mathcal{A} \to \Delta(\mathcal{S})$ is the transition function, $\gamma \in [0,1)$ is the discount factor, $R: \mathcal{S} \times \mathcal{A} \to \mathbb{R}$ is the reward function, and $\mu \in \Delta(\mathcal{S})$ is the initial state distribution. We denote by $\mathcal{D} = \{(s_i, a_i, r_i, s'_i)\}_{i=1}^{|\mathcal{D}|}$ the data buffer used for actor-critic training. In offline RL, $\mathcal{D}$ is fixed and pre-collected; in online RL, $\mathcal{D}$ is a replay buffer that is continually expanded through environment interaction. The goal is to learn a policy $\pi_\theta: \mathcal{S} \to \Delta(\mathcal{A})$ that maximizes the expected discounted return $J(\pi_\theta) = \mathbb{E}\left[\sum_{k=0}^{\infty} \gamma^k R(s_k, a_k)\right]$.
To optimize this objective from sampled transitions in $\mathcal{D}$, prior work commonly adopts behavior-regularized actor-critic frameworks~\citep{wu2019behavior, fujimoto2021minimalist, tarasov2023corl}. These methods are simple to implement yet empirically strong, and have been shown to achieve competitive performance on standard offline RL benchmarks~\citep{tarasov2023corl, pmlr-v202-yang23a, pmlr-v235-bai24d, park2025flowqlearning}. A minimalist instantiation minimizes the following critic and actor losses:
\begin{align}
    \mathcal{L}_{\mathrm{critic}}(\phi) &= \mathbb{E}_{(s,a,r,s') \sim \mathcal{D}} \left[\left(Q_\phi(s,a) - r - \gamma Q_{\bar\phi}(s', \pi_\theta(s'))\right)^2\right], \label{eq:critic_loss} \\
    \mathcal{L}_{\mathrm{actor}}(\theta) &= -\mathbb{E}_{s \sim \mathcal{D}} \left[Q_\phi(s, \pi_\theta(s))\right] + \alpha\, \mathbb{E}_{(s,a) \sim \mathcal{D}} F(a, \pi_\theta(s)), \label{eq:actor_loss}
\end{align}
where $Q_\phi$ is the critic parameterized by $\phi$, which estimates the expected return of taking action $a$ in state $s$, and $Q_{\bar{\phi}}$ is the corresponding target critic used for stable bootstrapping. The function $F(a, \pi_\theta(s))$ measures the discrepancy between the dataset action $a$ and the policy action $\pi_\theta(s)$; The actor loss maximizes the critic-estimated value of the policy action, regularized by a behavior-cloning penalty that limits deviation from the dataset action. Its optimization requires differentiating through the policy $\pi_\theta$, allowing the critic's action gradient $\nabla_a Q_\phi(s,a)$ to update the policy parameters. For one-step policies, such as Gaussian actors whose network output directly parameterizes the final action $a$, actor-critic training is well studied. However, for multi-step generative policies, the critic's action gradient must be backpropagated through the entire generation process, making policy optimization substantially more challenging, which we discuss next.

\subsection{Generative Control Policies}
Generative control policies represent the action distribution by transforming a simple noise source into an executable action through an iterative generation process conditioned on the state $s$. Two representative constructions of this paradigm are flow matching~\citep{lipman2022flow, liu2022flow, albergo2022building} and diffusion models~\citep{ho2020denoising, song2020score}. Their multi-step generation process gives policies enough expressiveness to model complex and multimodal action distributions, making them particularly attractive for robot learning from diverse demonstrations~\citep{chi2023diffusion_policy}. Flow matching parameterizes a state-conditioned policy with a time-dependent velocity field $v_\theta(x,t;\,s)$ that transports noise into actions. In the continuous-time view, this transport is described by
\begin{align}
    \mathrm{d}\hat{X}_t = v_\theta(\hat{X}_t, t;\, s)\,\mathrm{d}t, \qquad \hat{X}_0 \sim \mathcal{N}(0, I_d).
    \label{eq:flow_ode}
\end{align}
The velocity field is trained with the conditional flow matching objective~\citep{lipman2022flow}:
\begin{align}
    \mathcal{L}_{\mathrm{FM}}(\theta) =
    \mathbb{E}_{t \sim \mathcal{U}[0,1],\, z \sim \mathcal{N}(0,I),\, a \sim p_1(\cdot \mid s)}
    \left[\|v_\theta(x_t, t;\, s) - (a - z)\|^2\right],
    \label{eq:fm_loss}
\end{align}
where $ x_t = t\,a + (1-t)\,z$ is the noisy interpolation between a data action $a$ and noise $z$ at time $t$. At inference, the standard flow policy generates an action by discretizing the ODE into $K$ Euler steps from initial noise $x_0 \sim \mathcal{N}(0,I_d)$:
\begin{align}
    x_{k+1} = x_k + \tfrac{1}{K}\, v_\theta(x_k, t_k;\, s), \qquad a = x_K.
    \label{eq:euler_inference}
\end{align}
Thus, the same initial noise sample is refined along a single denoising trajectory until it becomes the final action. Diffusion policies also perform iterative denoising, gradually transforming noise into actions through a sequence of intermediate states. While they are often trained as score- or noise-prediction models with a prescribed noise schedule, suitable parameterizations can rewrite them as velocity-field models~\citep{song2020score, kingma2024understanding}. We adopt the flow-matching formulation because it gives a direct velocity-field view and typically enables faster generation with fewer steps. Nevertheless, \paper is a sampling-level redesign and can be applied beyond flow matching to iterative generative policies such as diffusion policies.

\subsection{The Instability Challenge of Applying RL on Generative Control Policies.}
When a policy is parameterized as a generative control policy, value-gradient optimization requires backpropagation through multiple denoising steps. Recall that rewards are assigned to state-action pairs, and the critic trained by Eq.~\eqref{eq:critic_loss} evaluates completed actions in the action space. In standard actor-critic training, $Q_\phi(s,a)$ assigns value to a final action $a$. However, in a generative control policy, this action is produced through an internal iterative denoising process rather than a single network evaluation. For a flow policy, the final action is the endpoint $a=x_K$ of the Euler sampler in Eq.~\eqref{eq:euler_inference}. Updating the velocity network with the actor objective therefore requires propagating the critic's action gradient through every denoising step. By the chain rule, the actor gradient can be written as
\begin{align}
    \label{eq:bptt}
    \nabla_\theta Q_\phi(s,a)
    =
    \sum_{k=0}^{K-1}
    \tfrac{1}{K}\,
    \nabla_a Q_\phi(s,a)
    \frac{\partial x_K}{\partial x_{k+1}}
    \frac{\partial v_\theta(x_k,t_k;\,s)}{\partial \theta},
    \qquad a=x_K.
\end{align}
This expression shows that each velocity update is driven by a gradient derived from the final action-value and is propagated backward through all subsequent denoising steps. As discussed in prior work, such direct value-gradient optimization is often unstable for multi-step generative policies~\citep{shi2026flowdpgdeterministicpolicygradient}. Existing methods therefore avoid the full BPTT problem either by removing $\nabla_a Q_\phi$ from the policy update or by shortening the generation horizon $K$~\citep{zhang2026reinflow,psenka2023learning,park2025flowqlearning}. In contrast, a key insight is that the chain fundamentally depends on the sampler, which determines the denoising path. We therefore next discuss how to construct a stable denoising path for actor updates.

\begin{algorithm}[t]
\caption{\paper: Actor-Critic Training}
\label{alg:vine}
\begin{algorithmic}[1]
\Function{Generate}{$s$} 
    \State $\hat{a}_0 \sim \mathcal{N}(0, I_d)$
    \For{$k = 0, \ldots, K - 1$} \hfill\textcolor{gray}{$\triangleright$ Iterative generation}
        \State $t_k \gets k/K$
        \State \textcolor{red}{\sout{$x_k \gets x_k + \tfrac{1}{K}\, v_\theta(x_k, t_k;\, s)$} \hfill$\triangleright$ Replace Euler Method}
        \State $z_k \sim \mathcal{N}(0, I_d)$
        \State $\hat{x}_k \gets t_k\, \hat{a}_{k} + (1 - t_k)\, z_k$
        \State $\hat{a}_{k+1} \gets \hat{x}_k + (1 - t_k)\, v_\theta(\hat{x}_k, t_k;\, s)$
    \EndFor
    \State \Return $\hat{a}_K$
\EndFunction
\While{not converged}
    \State Collect transitions with $\Call{Generate}{s}$ and add to $\dataset$ \hfill\textcolor{gray}{$\triangleright$ Optionally for online RL}
    \State Sample batch $\{(s, a, r, s')\} \sim \dataset$
    \State $a' \gets \Call{Generate}{s'}$ 
    \State Update $\phi$ to minimize $\E[(Q_\phi(s, a) - r - \gamma Q_{\bar\phi}(s', a'))^2]$ \hfill\textcolor{gray}{$\triangleright$ Train critic $Q_\phi$}
    \State $\hat{a}_K \gets \Call{Generate}{s}$
    \State Update $\theta$ to minimize $-Q_\phi(s, \hat{a}_K) + \alpha \| \hat{a}_K - a \|^2$ \hfill\textcolor{gray}{$\triangleright$ Train velocity $v_\theta$ via BPTT} 

\EndWhile
\State \Return policy $\pi_\theta(s) \equiv \Call{Generate}{s}$
\end{algorithmic}
\end{algorithm}

\section{\paper: Value-gradient Iterative Noise Exploration}
\label{sec:method_noise}
To make value-gradient propagation through the sampling chain more stable, we reconsider how the sampler transports one intermediate state to the next. The key observation is that flow-matching training adds noise to each target action by sampling many noise levels $t$ and noise draws $z$. The velocity field therefore learns to reach the same action from many possible paths, whereas the Euler ODE sampler in Eq.~\eqref{eq:euler_inference} follows only one of them. This gives us the freedom to construct a different sampling path. \paper uses this freedom to build an RL-oriented sampler: at each step, it reconstructs a noisy intermediate state around the current action estimate, then applies the same velocity field to produce a refined estimate. Recall that in standard flow-matching training, noise is introduced by interpolating between a ground-truth action and a Gaussian noise sample: $x_t=t\,a+(1-t)z$. If we have a current action estimate $\hat{a}_k$ predicted before the $k$ denoising step, the natural way to inject noise while staying consistent with flow matching is to use the same interpolation form:
\begin{equation}
    \hat{x}_{k}
    =
    t_{k}\,\hat{a}_k + (1 - t_{k})\,z_{k}, \quad z_{k} \sim \mathcal{N}(0, I_d).
    \label{eq:vine_interpolation}
\end{equation}
We denote the reconstructed interpolation state by $\hat{x}_k$ to distinguish it from the Euler state $x_k$, the noise $z_{k}$ is sampled independently at each step. We take $0 \leq t_1 < \cdots < t_K \leq 1$ to be a monotonically increasing time schedule, so that $\hat{x}_{k}$ is progressively dominated by the current action estimate $\hat{a}_k$ as $k$ grows. Note that $\hat{x}_{k}$ is a fresh noisy interpolation state used as the network input, rather than a state obtained in Eq.~\eqref{eq:euler_inference}. The remaining question is how to recover an action estimate from this noisy state during RL rollout, where the ground-truth endpoint is unavailable. From a probabilistic view, the flow-matching objective trains $v_\theta$ to predict the conditional velocity from a noisy interpolation state. Under the optimal velocity $v^\star$, the transformed quantity $x_t+(1-t)v^\star(x_t,t;\,s)$ estimates the posterior mean of the final action. Applying this closed form to $\hat{x}_{k}$ yields,
\begin{equation}
    \hat{a}_{k+1}=\hat{x}_{k}+(1 - t_{k})\,v_\theta(\hat{x}_{k}, t_{k};\,s),
    \label{eq:vine_step}
\end{equation}
Starting from an initial action estimate $\hat a_0\sim\mathcal{N}(0, I_d)$,
we recursively compute the interpolation state with Eq.\eqref{eq:vine_interpolation} and endpoint prediction with Eq.\eqref{eq:vine_step} so the final action becomes $a = \hat{a}_K$. This sampling method has two practical consequences. First, each network query remains compatible with the flow-matching training semantics: the input is a noisy interpolation state, and the velocity field is used in its form. Thus, \paper changes the sampling path without requiring a new generative objective or an auxiliary guidance model; we provide a more formal justification in Appendix~\ref{sec:method_theory}. Second, fresh noise is injected locally around the current action estimate at every step. This makes stochasticity act as structured local exploration during refinement, rather than letting a single initial noise sample determine the entire action trajectory as in Eq.~\eqref{eq:euler_inference}. Moreover, since $t_k$ is monotonically increasing, the noise weight $(1-t_k)$ decreases across steps, moving the sampler from coarse stochastic exploration toward finer deterministic refinement of the final action.

To understand the sampling process, we visualize this behavior with a toy simulation in noise space in Fig.~\ref{fig:toy_example}. Under behavior cloning, both the flow-matching and \paper are trained using the same flow-matching objective in Eq.~\eqref{eq:fm_loss}, while the DDPM baseline is trained with the standard diffusion loss. At inference, we sample many initial points from $\mathcal{N}(0,I)$ and trace their trajectories using three samplers: Euler, a DDPM sampler, and \paper. Euler trajectories form isolated petal-like paths with few bridges between modes. DDPM explores more broadly, but its stochastic transitions introduce endpoint-prediction error. In contrast, \paper traverses the full pentagon with bridging paths between neighboring modes, yielding broader state coverage while preserving accurate endpoints. This coverage is useful under critic-gradient fine-tuning, where the desired action may shift from one mode to another.
\section{Related Work}
\label{sec:related_work}
\textbf{Gaussian and one-step actors for RL.}
Classical actor-critic methods typically parameterize the policy with a Gaussian, deterministic, or residual actor, making value-gradient optimization straightforward because the critic's action gradient is backpropagated through a single network evaluation~\citep{wu2019behavior, fujimoto2021minimalist, tarasov2023corl,pmlr-v32-silver14}. This simplicity makes Gaussian-style actors strong baselines for offline RL, including behavior-regularized methods~\citep{pmlr-v97-fujimoto19a, NEURIPS2019_c2073ffa}. Recent work applies a similar principle to generative policies by shortening the sampling chain, using consistency models, shortcut models, or one-step flow policies before applying value-gradient optimization~\citep{ding2023consistency, chen2024diffusion, park2025flowqlearning, espinosa2025scaling, chen2025one}. These approaches improve stability by simplifying the policy computation graph, but they sacrifice the iterative refinement and expressive multimodal modeling that motivate generative control policies.

\textbf{RL fine-tuning of diffusion and flow policies.}
Diffusion and flow-matching policies have been widely adopted in robot learning and reinforcement learning because they can represent complex multimodal action distributions~\citep{chi2023diffusion_policy, ze20243d, pi02024, barreiros2026careful}. The central challenge is how to optimize these multi-step generative policies with a learned critic~\citep{zhu2024diffusionmodelsreinforcementlearning,uehara2024understandingreinforcementlearningbasedfinetuning}. Existing methods differ mainly in where the value signal is applied relative to the denoising chain. \emph{Around-the-chain methods} avoid direct BPTT through the full sampler by constructing objectives around the denoising process. Advantage- or value-weighted methods avoid critic action gradients by reweighting generative-model objectives with scalar critic values~\citep{zhang2025energyweighted, edp_kang2023, qvpo_ding2024}. Other methods convert value, energy, or guidance signals into denoising objectives or guidance terms, steering generation without directly optimizing through the complete sampling chain~\citep{psenka2023learning, fang2025diffusionactorcriticformulatingconstrained, lu2023contrastive, frans2025diffusion}. Adjoint matching~\citep{domingo-enrich2025adjoint} transports terminal gradients backward to provide step-wise supervision, and QAM~\citep{qam2026} applies this idea to offline RL. Related variants make different approximation and complexity trade-offs~\citep{uehara2024fine, bergmeister2025ram, guo2026deterministic, shin2026efficient}. \emph{Outside-the-chain methods} improve actions after or outside the base sampling process through value-guided selection~\citep{hansen2023idql, dong2025expo}, residual corrections in action or noise space~\citep{ankile2024imitation, yuan2024policy, wagenmaker2025steering}, or gradient-based action editing~\citep{mark2024policy}. These methods can improve sampled actions, but the value signal acts outside the internal denoising dynamics. \emph{Through-the-chain methods} directly differentiate through the full denoising process~\citep{dql_wang2023, diffcps_he2023, entropydql_zhang2024,zhang2026sacflowsampleefficientreinforcement}. This is the most direct way to use critic action gradients, but prior work has found it unstable for expressive multi-step policies. In contrast, \paper keeps the full iterative chain and supports stable value-gradient BPTT by redesigning the noise injection process, rather than bypassing the chain, shortening it, or applying value signals only outside it.
\begin{table*}[t]
    \centering
    \makebox[\textwidth]{\scalebox{0.63}{
    \begin{tabular}{lccccccccccccccc}
\toprule
 & \multicolumn{1}{c}{\textbf{Gaussian}}
 & \multicolumn{4}{c}{\textbf{Around the Chain}}
 & \multicolumn{7}{c}{\textbf{Outside the Chain}}
 & \multicolumn{3}{c}{\textbf{Through the Chain}} \\
\cmidrule(lr){2-2}\cmidrule(lr){3-6}\cmidrule(lr){7-13}\cmidrule(lr){14-16}
\textbf{Task Category}
 & \texttt{ReBRAC}
 & \texttt{FQL} & \texttt{FAWAC} & \texttt{IFQL} & \texttt{QAM} & \texttt{DAC} & \texttt{QSM} & \texttt{CGQL} & \texttt{CGQL-M} & \texttt{CGQL-L}
 & \texttt{DSRL} & \texttt{FEdit}
 & \texttt{FBRAC} & \texttt{BAM} & \texttt{\textcolor{orange}{VINE}} \\
\midrule
\texttt{antmaze-large}
 & $\overset{\scaleto{[94,95]}{3pt}}{94}$
 & $\overset{\scaleto{[72,79]}{3pt}}{76}$ & $\overset{\scaleto{[15,19]}{3pt}}{17}$ & $\overset{\scaleto{[32,39]}{3pt}}{36}$ & $\overset{\scaleto{[78,84]}{3pt}}{81}$ & $\overset{\scaleto{[86,90]}{3pt}}{88}$ & $\overset{\scaleto{[89,93]}{3pt}}{91}$ & $\overset{\scaleto{[73,80]}{3pt}}{76}$ & $\overset{\scaleto{[68,73]}{3pt}}{71}$ & $\overset{\scaleto{[62,67]}{3pt}}{65}$
 & $\overset{\scaleto{[56,66]}{3pt}}{61}$ & $\overset{\scaleto{[53,62]}{3pt}}{58}$
 & $\overset{\scaleto{[1,4]}{3pt}}{2}$ & $\overset{\scaleto{[82,85]}{3pt}}{84}$ & $\overset{\scaleto{[98,100]}{3pt}}{\mathbf{99}}$ \\
\texttt{antmaze-giant}
 & $\overset{\scaleto{[53,60]}{3pt}}{57}$
 & $\overset{\scaleto{[0,0]}{3pt}}{0}$ & $\overset{\scaleto{[0,0]}{3pt}}{0}$ & $\overset{\scaleto{[0,2]}{3pt}}{1}$ & $\overset{\scaleto{[14,22]}{3pt}}{18}$ & $\overset{\scaleto{[11,21]}{3pt}}{16}$ & $\overset{\scaleto{[11,17]}{3pt}}{15}$ & $\overset{\scaleto{[0,2]}{3pt}}{0}$ & $\overset{\scaleto{[1,8]}{3pt}}{4}$ & $\overset{\scaleto{[1,6]}{3pt}}{3}$
 & $\overset{\scaleto{[1,4]}{3pt}}{3}$ & $\overset{\scaleto{[1,3]}{3pt}}{2}$
 & $\overset{\scaleto{[0,0]}{3pt}}{0}$ & $\overset{\scaleto{[0,2]}{3pt}}{1}$ & $\overset{\scaleto{[68,83]}{3pt}}{\mathbf{76}}$ \\
\texttt{humanoidmaze-medium}
 & $\overset{\scaleto{[65,74]}{3pt}}{69}$
 & $\overset{\scaleto{[63,73]}{3pt}}{68}$ & $\overset{\scaleto{[22,26]}{3pt}}{24}$ & $\overset{\scaleto{[85,87]}{3pt}}{86}$ & $\overset{\scaleto{[64,69]}{3pt}}{67}$ & $\overset{\scaleto{[81,85]}{3pt}}{83}$ & $\overset{\scaleto{[80,86]}{3pt}}{83}$ & $\overset{\scaleto{[57,62]}{3pt}}{60}$ & $\overset{\scaleto{[40,43]}{3pt}}{42}$ & $\overset{\scaleto{[57,67]}{3pt}}{62}$
 & $\overset{\scaleto{[48,57]}{3pt}}{53}$ & $\overset{\scaleto{[20,23]}{3pt}}{22}$
 & $\overset{\scaleto{[37,41]}{3pt}}{39}$ & $\overset{\scaleto{[58,62]}{3pt}}{60}$ & $\overset{\scaleto{[79,93]}{3pt}}{\mathbf{87}}$ \\
\texttt{humanoidmaze-large}
 & $\overset{\scaleto{[15,20]}{3pt}}{17}$
 & $\overset{\scaleto{[7,11]}{3pt}}{9}$ & $\overset{\scaleto{[0,0]}{3pt}}{0}$ & $\overset{\scaleto{[21,27]}{3pt}}{24}$ & $\overset{\scaleto{[9,14]}{3pt}}{11}$ & $\overset{\scaleto{[0,0]}{3pt}}{0}$ & $\overset{\scaleto{[9,11]}{3pt}}{10}$ & $\overset{\scaleto{[4,5]}{3pt}}{5}$ & $\overset{\scaleto{[3,8]}{3pt}}{6}$ & $\overset{\scaleto{[5,8]}{3pt}}{6}$
 & $\overset{\scaleto{[2,5]}{3pt}}{3}$ & $\overset{\scaleto{[2,3]}{3pt}}{3}$
 & $\overset{\scaleto{[0,0]}{3pt}}{0}$ & $\overset{\scaleto{[4,8]}{3pt}}{5}$ & $\overset{\scaleto{[36,56]}{3pt}}{\mathbf{46}}$ \\
\texttt{scene-sparse}
 & $\overset{\scaleto{[61,69]}{3pt}}{65}$
 & $\overset{\scaleto{[77,80]}{3pt}}{78}$ & $\overset{\scaleto{[35,41]}{3pt}}{38}$ & $\overset{\scaleto{[83,85]}{3pt}}{84}$ & $\overset{\scaleto{[96,98]}{3pt}}{97}$ & $\overset{\scaleto{[65,70]}{3pt}}{68}$ & $\overset{\scaleto{[84,87]}{3pt}}{86}$ & $\overset{\scaleto{[36,40]}{3pt}}{38}$ & $\overset{\scaleto{[72,76]}{3pt}}{74}$ & $\overset{\scaleto{[85,91]}{3pt}}{88}$
 & $\overset{\scaleto{[99,100]}{3pt}}{\mathbf{99}}$ & $\overset{\scaleto{[60,65]}{3pt}}{62}$
 & $\overset{\scaleto{[43,57]}{3pt}}{50}$ & $\overset{\scaleto{[97,99]}{3pt}}{98}$ & $\overset{\scaleto{[60,85]}{3pt}}{73}$ \\
\texttt{puzzle-3x3-sparse}
 & $\overset{\scaleto{[73,84]}{3pt}}{79}$
 & $\overset{\scaleto{[60,78]}{3pt}}{70}$ & $\overset{\scaleto{[2,3]}{3pt}}{3}$ & $\overset{\scaleto{[100,100]}{3pt}}{\mathbf{100}}$ & $\overset{\scaleto{[99,100]}{3pt}}{\mathbf{100}}$ & $\overset{\scaleto{[62,75]}{3pt}}{68}$ & $\overset{\scaleto{[49,57]}{3pt}}{53}$ & $\overset{\scaleto{[39,55]}{3pt}}{48}$ & $\overset{\scaleto{[100,100]}{3pt}}{\mathbf{100}}$ & $\overset{\scaleto{[83,96]}{3pt}}{90}$
 & $\overset{\scaleto{[82,92]}{3pt}}{87}$ & $\overset{\scaleto{[98,100]}{3pt}}{99}$
 & $\overset{\scaleto{[0,1]}{3pt}}{0}$ & $\overset{\scaleto{[48,64]}{3pt}}{56}$ & $\overset{\scaleto{[93,97]}{3pt}}{95}$ \\
\texttt{cube-double}
 & $\overset{\scaleto{[8,10]}{3pt}}{9}$
 & $\overset{\scaleto{[43,49]}{3pt}}{46}$ & $\overset{\scaleto{[2,2]}{3pt}}{2}$ & $\overset{\scaleto{[10,12]}{3pt}}{11}$ & $\overset{\scaleto{[62,66]}{3pt}}{64}$ & $\overset{\scaleto{[33,36]}{3pt}}{35}$ & $\overset{\scaleto{[53,58]}{3pt}}{56}$ & $\overset{\scaleto{[36,41]}{3pt}}{38}$ & $\overset{\scaleto{[39,43]}{3pt}}{41}$ & $\overset{\scaleto{[43,47]}{3pt}}{45}$
 & $\overset{\scaleto{[72,76]}{3pt}}{\mathbf{74}}$ & $\overset{\scaleto{[37,43]}{3pt}}{40}$
 & $\overset{\scaleto{[0,0]}{3pt}}{0}$ & $\overset{\scaleto{[44,50]}{3pt}}{47}$ & $\overset{\scaleto{[1,6]}{3pt}}{4}$ \\
\texttt{cube-triple}
 & $\overset{\scaleto{[0,1]}{3pt}}{1}$
 & $\overset{\scaleto{[2,4]}{3pt}}{3}$ & $\overset{\scaleto{[0,0]}{3pt}}{0}$ & $\overset{\scaleto{[0,0]}{3pt}}{0}$ & $\overset{\scaleto{[3,4]}{3pt}}{3}$ & $\overset{\scaleto{[3,6]}{3pt}}{5}$ & $\overset{\scaleto{[3,4]}{3pt}}{3}$ & $\overset{\scaleto{[7,9]}{3pt}}{\mathbf{8}}$ & $\overset{\scaleto{[7,9]}{3pt}}{\mathbf{8}}$ & $\overset{\scaleto{[7,9]}{3pt}}{\mathbf{8}}$
 & $\overset{\scaleto{[1,2]}{3pt}}{1}$ & $\overset{\scaleto{[2,3]}{3pt}}{2}$
 & $\overset{\scaleto{[0,1]}{3pt}}{0}$ & $\overset{\scaleto{[2,5]}{3pt}}{3}$ & $\overset{\scaleto{[0,1]}{3pt}}{1}$ \\
\midrule
\texttt{all} (40 tasks)
 & $\overset{\scaleto{[37,60]}{3pt}}{49}$
 & $\overset{\scaleto{[31,56]}{3pt}}{44}$ & $\overset{\scaleto{[6,16]}{3pt}}{10}$ & $\overset{\scaleto{[31,55]}{3pt}}{43}$ & $\overset{\scaleto{[42,68]}{3pt}}{55}$ & $\overset{\scaleto{[33,57]}{3pt}}{45}$ & $\overset{\scaleto{[37,62]}{3pt}}{50}$ & $\overset{\scaleto{[24,44]}{3pt}}{34}$ & $\overset{\scaleto{[30,56]}{3pt}}{43}$ & $\overset{\scaleto{[34,57]}{3pt}}{46}$
 & $\overset{\scaleto{[35,61]}{3pt}}{48}$ & $\overset{\scaleto{[24,48]}{3pt}}{36}$
 & $\overset{\scaleto{[5,20]}{3pt}}{12}$ & $\overset{\scaleto{[32,56]}{3pt}}{44}$ & $\overset{\scaleto{[55,65]}{3pt}}{\mathbf{60}}$ \\
\bottomrule
\end{tabular}}}
   \caption{\textbf{Offline RL results on OGBench.} We follow the QAM evaluation protocol~\cite{qam2026}. Around the Chain denotes methods that update the generative control policy with gradient information but avoid BPTT through the full sampler; Outside the Chain denotes methods that freeze the generative policy and optimize actions externally; Through the Chain denotes methods that directly perform BPTT through the sampler. \paper achieves better performance on 40 tasks in total.}
    \label{tab:offline_results}
\end{table*}
\section{Experiments}
\label{sec:experiments}
We conduct experiments to study how well \paper optimizes generative control policies on both benchmark tasks and real-world manipulation. Concretely, we ask three research questions: (1) Can \paper optimize policies effectively compared to representative offline RL methods on OGBench? (2) Can \paper make flow-matching stable for real-world online RL? (3) Why VINE helps?
\subsection{Experimental Setup}
\textbf{Offline RL setting.}
We first evaluate \paper in the offline RL setting on eight domains from OGBench, each containing five tasks: antmaze-large, antmaze-giant, humanoidmaze-medium, humanoidmaze-large, scene-sparse, puzzle-3x3-sparse, cube-double, and cube-triple. The antmaze and humanoidmaze domains test long-horizon navigation from diverse offline trajectories, while the scene, puzzle, and cube domains require solving sparse-reward manipulation-style tasks. For offline RL baselines, we compare against representative Gaussian, flow, and diffusion policy methods, grouped by how they propagate value-gradient updates through or around the sampling chain. These include ReBRAC~\cite{tarasov2023corl}, FQL and FAWAC~\cite{park2025flowqlearning}, IFQL~\cite{hansen2023idql}, QAM~\cite{qam2026}, DAC~\cite{fang2025diffusionactorcriticformulatingconstrained}, QSM~\cite{psenka2023learning}, CGQL variants~\cite{lu2023contrastive}, DSRL~\cite{wagenmaker2025steering}, FEdit~\cite{dong2025expo}, FBRAC~\cite{park2025flowqlearning}, and BAM~\cite{domingo-enrich2025adjoint,qam2026}.

\textbf{Real-world online RL setting.}
We further evaluate \paper on a real-world socket insertion task to assess whether stable value-gradient optimization can transfer to physical robotic manipulation. In this task, the robot must pick up a plug from the table and insert it into a socket, requiring precise contact-rich control under tight insertion tolerances. We evaluate each method over 20 trials and report the success rate, wall-clock fine-tuning time, and the fraction of trajectories requiring human intervention. We compare against representative real-world reinforcement learning baselines, including SAC-Flow~\cite{zhang2026sacflowsampleefficientreinforcement}, Hil-SERL~\cite{luo2025precise}, DSRL~\cite{wagenmaker2025steering}, RLT~\cite{xu2026rl}, and EXPO~\cite{dong2025expo}. All methods receive images and proprioceptive states as inputs and directly predict low-level actions or action chunks. We divide the baselines into two categories according to their training paradigms. DSRL, RLT, and EXPO leverage pretrained policies. Specifically, we initialize these methods from $\pi_{0.5}$ and perform BC warm-up using 15 pre-collected human demonstrations before online RL. In contrast, Hil-SERL, SAC-Flow, and VINE are trained without pretrained initialization. For Hil-SERL, SAC-Flow, and VINE, we use a ResNet-10 encoder to extract visual features. SAC-Flow follows flow-matching objective but employs a transformer-based action head, whereas Hil-SERL and VINE use a three-layer MLP action head for single action prediction. For fair comparison, SAC-Flow, and VINE all use 10 denoising steps during action generation.

\subsection{Can \paper optimize effectively compared to representative methods on OGBench?}
Table~\ref{tab:offline_results} reports results on 40 OGBench tasks across eight domains. Following QAM \cite{qam2026}, we train each task with 12 seeds and report the mean performance with 95\% confidence intervals computed by bootstrapping with 5000 samples. \paper achieves the best aggregate score, with especially large gains on long-horizon navigation domains such as \texttt{antmaze-giant} and \texttt{humanoidmaze-large}. These results suggest that \paper can stably optimize expressive generative policies with value gradients in offline RL.
\subsection{Can \paper make flow-matching stable for real-world online RL?}
We evaluate \paper on a real-world socket insertion task as shown in Figure~\ref{fig:failure_case}, where the robot must pick up a plug and insert it into a fixed socket under tight contact tolerances. We compare against baselines under human-in-the-loop real-world RL setting~\citep{luo2025precise,xu2026rl}. Human-in-the-loop means human operator intervenes and provides corrections during autonomous execution. We report success rate, training time, and human-intervention ratio in Table~\ref{tab:real}. \paper achieves the highest success rate with the shortest fine-tuning time and the lowest human-intervention ratio among the online methods.

\begin{figure}
    \centering
    \includegraphics[width=1\linewidth]{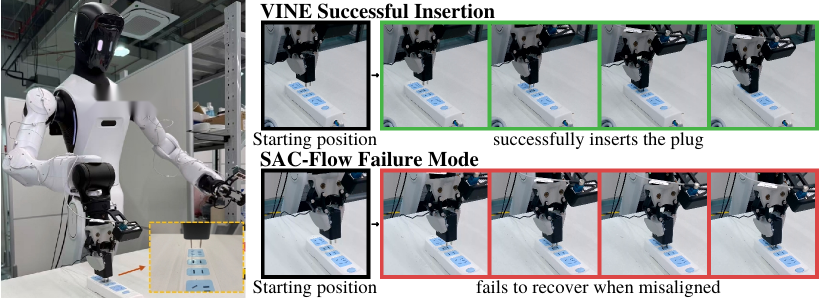}
    \caption{\textbf{Real-world online RL setting}. \paper learns an policy that insert successfully from any location. SAC-Flow fails to learn a successful policy during online RL.}
    \label{fig:failure_case}
\end{figure}
\begin{table}
\centering
\small
\caption{\textbf{Online RL results on a real-world human-in-the-loop manipulation task.}
We evaluate plug insertion and report success rate over 20 trials, online RL fine-tuning time, and the fraction of online rollout steps requiring human intervention.}
\label{tab:real}
\vspace{-2mm}
\begin{tabular}{lccccccc}
\toprule
\textbf{Plug Insertion} &
\begin{tabular}[c]{@{}c@{}}
BC init.
\end{tabular} &
DSRL &
RLT &
EXPO &
SAC-Flow &
Hil-SERL &
\textcolor{orange}{VINE} \\
\midrule
\begin{tabular}[c]{@{}l@{}}
Success rate ($\uparrow$) \\
Time ($\downarrow$)
\end{tabular} &
\begin{tabular}[c]{@{}c@{}}10/20\\--\end{tabular} &
\begin{tabular}[c]{@{}c@{}}17/20\\50 min\end{tabular} &
\begin{tabular}[c]{@{}c@{}}17/20\\50 min\end{tabular} &
\begin{tabular}[c]{@{}c@{}}19/20\\50 min\end{tabular} &
\begin{tabular}[c]{@{}c@{}}12/20\\20 min\end{tabular} &
\begin{tabular}[c]{@{}c@{}}16/20\\20 min\end{tabular} &
\begin{tabular}[c]{@{}c@{}}\textbf{20/20}\\\textbf{20 min}\end{tabular} \\
Human-intervention ($\downarrow$) &
-- &
-- &
29.7\% &
56.3\% &
74.6\% &
55.9\% &
\textbf{19.2\%} \\
\bottomrule
\end{tabular}
\end{table}
\vspace{-3mm}
\begin{figure}
\centering
\includegraphics[width=0.9\linewidth]{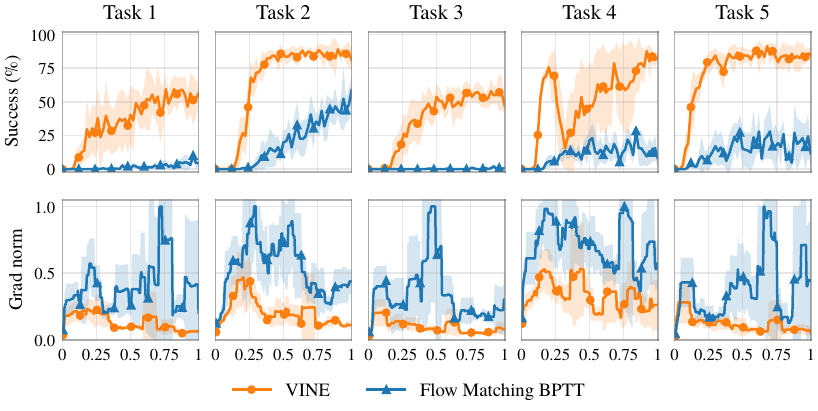}
\vspace{-5mm}
\caption{\textbf{Sampling ablation.} We compare VINE against the vanilla Euler solver for flow-matching on OGBench. VINE consistently achieves higher success rates and more stable backpropagated gradients.}
\label{fig:ablation-stochastic-actor-update}
\end{figure}

\subsection{Why VINE Helps?}
\label{sec:method_iteration}
We use ablations to isolate two design choices in \paper: per-step stochastic interpolation and iterative endpoint-prediction. Both are evaluated under the same actor-critic training protocol as the full method, changing only the component under study.

\textbf{VINE stabilizes value-gradient backpropagation.}
We compare the standard Euler sampler used by flow matching with VINE. The Euler sampler samples noise only once at initialization and follows a fixed deterministic denoising trajectory throughout inference, whereas VINE reconstructs a noisy interpolation state by injecting fresh Gaussian noise, at every denoising step. As shown in Fig.~\ref{fig:ablation-stochastic-actor-update}, VINE consistently achieves substantially higher success rates across all five OGbench tasks. Moreover, VINE maintains significantly smaller and more stable backpropagated gradients throughout training, whereas the Euler sampler exhibits frequent large gradient spikes that correlate with unstable optimization. These results demonstrate that VINE stabilizes end-to-end value-gradient propagation through the denoising chain, leading to both more robust optimization and stronger policy performance.

\textbf{Iterative generation improves action quality from coarse to fine.}
We evaluate different intermediate denoising steps with $K = 1,4,6,8,10$. Figure~\ref{fig:iterative_refinement} shows that early intermediate actions are insufficient for solving the long-horizon AntMaze task. With only 1 or 4 denoising steps, the agent largely remains near the start region and fails to reach the goal. VINE does not merely produce a good action in the first few iterations; instead, the later refinement steps are essential for transforming an initially poor action proposal into a task-completing action.
\begin{figure*}
    \centering
    \includegraphics[width=0.9\linewidth]{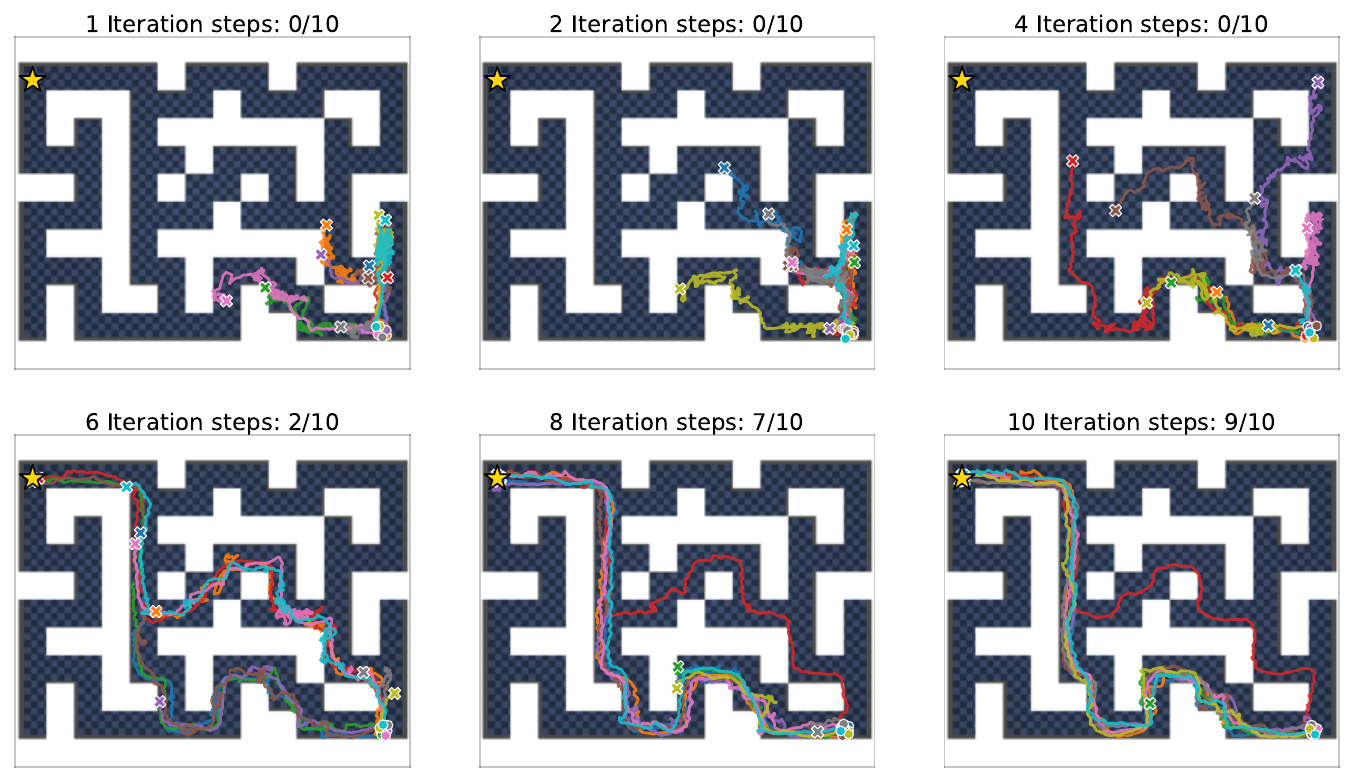}
    \vspace{-4mm}
    \caption{\textbf{Preserving the Iterative Refinement.} We generate the final action using different numbers of denoising steps $K$ and evaluate the resulting policies on the AntMaze task. As the number of refinement iterations increases, the generated action quality improves, leading to higher task success rates.}
    \label{fig:iterative_refinement}
\end{figure*}

\vspace{-2mm}
\section{Conclusion}
\label{sec:conclusion}
In this paper, we identify the training instability of value-gradient RL with expressive generative policies originates in part from the sampler and proposed \paper, an RL-oriented sampling method that enables stable value-gradient optimization toward expressive policies. \paper remains compatible with pretrained flow-matching policies. Empirically, \paper achieves stable policy improvement with a $10$-step denoising process and consistently outperforms state-of-the-art baselines on both OGBench and real-robot task. A promising future direction is to extend \paper to large foundation models and adapt variant RL algorithms for broader validation.

\newpage

\bibliography{main}
\newpage
\appendix
\section{Appendix}

\subsection{Compatibility with Pretrained Flow Policies}
\label{sec:method_theory}
A natural question is whether \paper can leverage an existing flow-matching pretrained policy. The following theorem justifies that the same iterative procedure would also apply if a pretrained FM field were used as the velocity model.

\begin{theorem}[FM-VINE Consistency]
\label{thm:fm_vine}
Let $v^*(x, t;\, s)$ be the pointwise minimizer of the conditional flow matching loss $\mathcal{L}_{\mathrm{FM}}$ (Eq.~\ref{eq:fm_loss}).
Then for any query point $x \in \mathbb{R}^d$ and time $t \in (0,1)$,
\begin{equation}
    x + (1-t)\, v^*(x, t;\, s) = \mathbb{E}[a \mid x_t = x,\, t,\, s],
\end{equation}
regardless of how $x$ was constructed.
In particular, our endpoint prediction $\hat{a}_k = \hat{x}_k + (1-t_k)\, v^*(\hat{x}_k, t_k;\, s) = \mathbb{E}[a \mid \hat{x}_k, t_k, s]$ is the Bayes-optimal action estimate given the noisy intermediate $\hat{x}_k$, even though $\hat{x}_k$ was not generated by the FM ODE.
\end{theorem}
This theorem implies that a flow-matching pretrained policy can be used as the velocity field in \paper---for both inference (using VINE's iterative generation) and post-training (fine-tuning with value gradients via BPTT )without changing the generation rule.
\begin{corollary}[Contraction under Deterministic VINE]
\label{cor:contraction}
Suppose $p(a \mid s) = \mathcal{N}(\mu_s, \sigma_a^2 I)$.
Under deterministic VINE ($z_k = 0$, $\hat{a}_0 = 0$) with optimal velocity $v^*$, the action estimates satisfy the recursion $\hat{a}_k = \lambda_k\, \hat{a}_{k-1} + (1-\lambda_k)\, \mu_s$ with contraction factor
\begin{equation}
    \lambda_k = \frac{t_k^2 \sigma_a^2}{t_k^2 \sigma_a^2 + (1-t_k)^2} \in (0,1),
\end{equation}
giving $\|\hat{a}_k - \mu_s\| = \prod_{j=1}^{k}\lambda_j \cdot \|\mu_s\| \to 0$.
\end{corollary}

Since $\lambda_k < 1$ for all $t_k < 1$, deterministic VINE contracts monotonically to $\mu_s$.
The bulk of convergence occurs at early steps where $t_k$ is small and $(1-t_k)^2$ dominates the denominator, making $\lambda_k \approx 0$.

\subsection{Proof of Theorem~\ref{thm:fm_vine}}
\begin{proof}
Under the conditional flow matching objective with linear interpolation paths, the training input at time $t$ is $x_t = (1-t)\,z + t\,a$, where $z \sim \mathcal{N}(0, I)$ and $a \sim p(a \mid s)$.
The conditional velocity target is $a - z = \frac{a - x_t}{1-t}$.

The FM loss is a pointwise squared error:
\begin{align}
    \mathcal{L}_{\mathrm{FM}}(\theta) = \mathbb{E}_{t, z, a}\left[\left\|v_\theta(x_t, t;\, s) - \frac{a - x_t}{1-t}\right\|^2\right].
\end{align}
For any fixed query point $x$ and time $t$, the pointwise minimizer of the squared loss is the conditional expectation:
\begin{align}
    v^*(x, t;\, s) = \mathbb{E}\left[\frac{a - x_t}{1-t}\,\bigg|\, x_t = x,\, t,\, s\right] = \frac{\mathbb{E}[a \mid x_t = x, t, s] - x}{1-t}.
\end{align}
Therefore:
\begin{align}
    x + (1-t)\, v^*(x, t;\, s) = x + \mathbb{E}[a \mid x_t = x, t, s] - x = \mathbb{E}[a \mid x_t = x, t, s].
\end{align}
This holds for any $x \in \mathbb{R}^d$, regardless of its origin.
In particular, setting $x = \hat{x}_k = t_k\, \hat{a}_{k-1} + (1-t_k)\, z_k$ gives
\begin{align}
    \hat{a}_k = \hat{x}_k + (1-t_k)\, v^*(\hat{x}_k, t_k;\, s) = \mathbb{E}[a \mid x_{t_k} = \hat{x}_k, t_k, s],
\end{align}
which is the Bayes-optimal action estimate under squared error loss, conditioned on the noisy intermediate $\hat{x}_k$.
\end{proof}

\subsection{Proof of Corollary~\ref{cor:contraction}}

\begin{proof}
Let $p(a \mid s) = \mathcal{N}(\mu_s, \sigma_a^2 I)$ and consider deterministic VINE with $z_k = 0$ and $\hat{a}_0 = 0$.
The noisy intermediate is $\hat{x}_k = t_k\, \hat{a}_{k-1}$.

Under the linear interpolation $x_t = (1-t)\,z + t\,a$ with $z \sim \mathcal{N}(0, I)$ and $a \sim \mathcal{N}(\mu_s, \sigma_a^2 I)$, the joint $(x_t, a)$ is Gaussian.
The conditional posterior is:
\begin{align}
    p(a \mid x_t = x, t) = \mathcal{N}\left(\frac{t\,\sigma_a^2\, x + (1-t)^2\, \mu_s}{t^2\sigma_a^2 + (1-t)^2},\; \Sigma_{\mathrm{post}}\right),
\end{align}
where we used $\mathrm{Var}[x_t] = t^2\sigma_a^2 + (1-t)^2$ (isotropic components) and the standard linear-Gaussian conditioning formula.
The posterior covariance $\Sigma_{\mathrm{post}}$ is irrelevant for the mean.

By Theorem~\ref{thm:fm_vine}, endpoint prediction with optimal velocity gives:
\begin{align}
    \hat{a}_k = \mathbb{E}[a \mid x_{t_k} = \hat{x}_k, t_k] = \frac{t_k\,\sigma_a^2\, \hat{x}_k + (1-t_k)^2\, \mu_s}{t_k^2\sigma_a^2 + (1-t_k)^2}.
\end{align}
Substituting $\hat{x}_k = t_k\, \hat{a}_{k-1}$:
\begin{align}
    \hat{a}_k = \frac{t_k^2\,\sigma_a^2}{t_k^2\sigma_a^2 + (1-t_k)^2}\, \hat{a}_{k-1} + \frac{(1-t_k)^2}{t_k^2\sigma_a^2 + (1-t_k)^2}\, \mu_s.
\end{align}
Define $\lambda_k \coloneqq \frac{t_k^2\sigma_a^2}{t_k^2\sigma_a^2 + (1-t_k)^2}$.
Then $\hat{a}_k = \lambda_k\, \hat{a}_{k-1} + (1 - \lambda_k)\, \mu_s$.

Since $0 < \lambda_k < 1$ for all $t_k \in (0,1)$, this is a contraction toward $\mu_s$.
With $\hat{a}_0 = 0$, the error evolves as:
\begin{align}
    \hat{a}_k - \mu_s = \lambda_k\,(\hat{a}_{k-1} - \mu_s) = \cdots = \left(\prod_{j=1}^k \lambda_j\right)(\hat{a}_0 - \mu_s) = -\left(\prod_{j=1}^k \lambda_j\right) \mu_s.
\end{align}
Hence $\|\hat{a}_k - \mu_s\| = \prod_{j=1}^k \lambda_j \cdot \|\mu_s\| \to 0$ since each $\lambda_j < 1$.
\end{proof}

\subsection{Implementation Details}
In this section, we describe the full implementation details of VINE.
We use a step count of $K{=}10$ across all tasks, with the fixed MIP grid
$t \in \{0.0, 0.1, \ldots, 0.9\}$. Following the implementations in FQL, we train two Q functions to improve stability.
We take the mean of the two Q values for the Q loss term in the actor objective. For the critic target, we use the minimum of the two Q values (clipped double Q-learning) on all reported OGBench domains. The actor is an MLP with hidden sizes $[512, 512, 512, 512]$ and GELU activations
\citep{hendrycks2016gelu}. The critic is a twin ResNet value network with width $512$ (depth $4$ by default; depth $2$ on \texttt{humanoidmaze} and \texttt{cube} domains), with layer normalization \citep{ba2016layer} to stabilize training. This critic backbone matches our FQL baseline for fair comparison. We train offline VINE with 1M gradient steps for state-based OGBench tasks, and evaluate the agent every 20k steps using 50 episodes. We report the average success rates following the official evaluation scheme~\cite{qam2026},

\textbf{Hyperparameters.} We refer to Tables~\ref{tab:fql_hyperparameters} and~\ref{tab:fql_bc_alpha} for the complete list of hyperparameters.

\begin{table}
\centering
\caption{Hyperparameters for VINE.}
\label{tab:fql_hyperparameters}
\begin{tabular}{@{}p{0.40\linewidth}p{0.6\linewidth}@{}}
\toprule
\textbf{Hyperparameter} & \textbf{Value} \\
\midrule
Learning rate & 0.0003 \\
Optimizer & Adam \citep{kingma2015adam} \\
Gradient steps & 1000000 (default) \\
Minibatch size & 512 \\
Actor MLP dimensions & $[512, 512, 512, 512]$ \\
Critic architecture & ResNet (width $512$; depth $4$ / $2$) \\
Nonlinearity & GELU \citep{hendrycks2016gelu} \\
Target network smoothing coefficient & 0.005 \\
Discount factor $\gamma$ & 0.995 \\
Flow steps $K$ & 10 \\
MIP time grid & $\{0.0, 0.1, \ldots, 0.9\}$ \\
Clipped double Q-learning & True \\
Actor noise std (train) & 1.0 \\
BC coefficient $\alpha$ & Table~\ref{tab:fql_bc_alpha} \\
\bottomrule
\end{tabular}
\end{table}

\begin{table}[t]
\centering
\caption{VINE BC coefficient $\alpha$ for each task.}
\label{tab:fql_bc_alpha}
\small
\begin{tabular}{@{}p{0.5\linewidth}p{0.3\linewidth}@{}}
\toprule
\textbf{Task} & \textbf{VINE $\alpha$} \\
\midrule
\texttt{antmaze-large} & 10 \\

\texttt{antmaze-giant} & 10 \\

\texttt{humanoidmaze-medium} & 30 \\

\texttt{humanoidmaze-large} & 20 \\

\texttt{scene-sparse} & 300 \\

\texttt{puzzle-3x3-sparse} & 1000 \\

\texttt{cube-double} & 300 \\
\texttt{cube-triple} & 300 \\
\bottomrule
\end{tabular}
\end{table}

\begin{figure*}[t]
\centering
\setlength{\tabcolsep}{2pt}
\renewcommand{\arraystretch}{0.85}

\begin{tabular}{c}
\scriptsize\texttt{antmaze-large} \\
\includegraphics[width=0.95\textwidth]{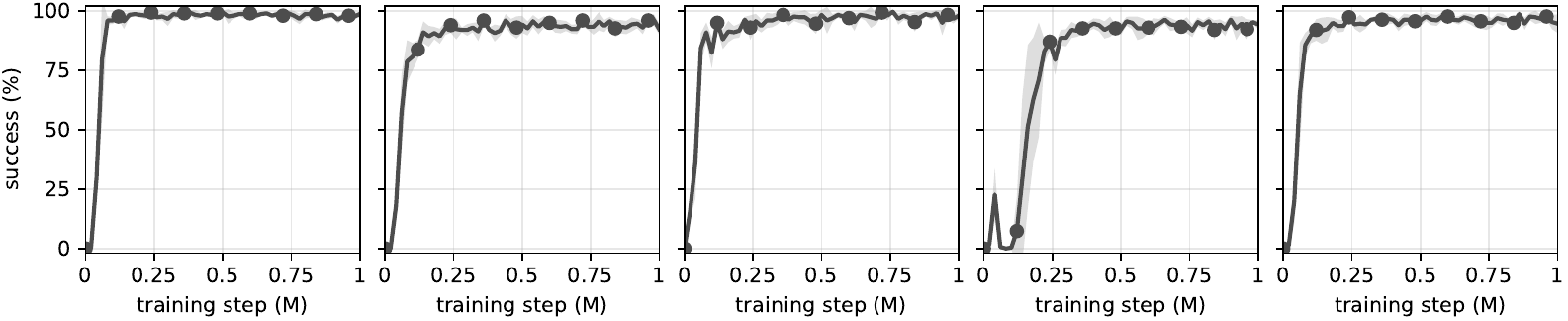} \\[6pt]
\scriptsize\texttt{antmaze-giant} \\
\includegraphics[width=0.95\textwidth]{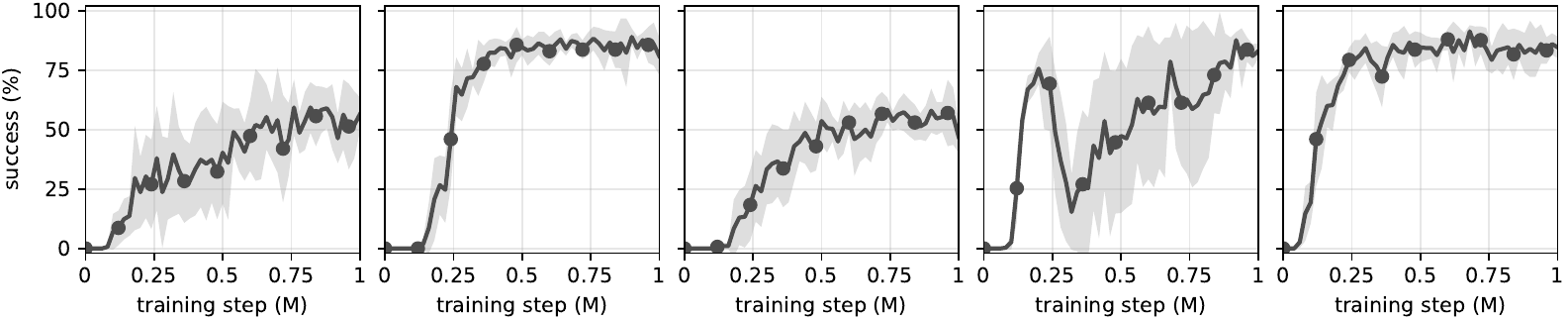} \\[6pt]
\scriptsize\texttt{humanoidmaze-medium} \\
\includegraphics[width=0.95\textwidth]{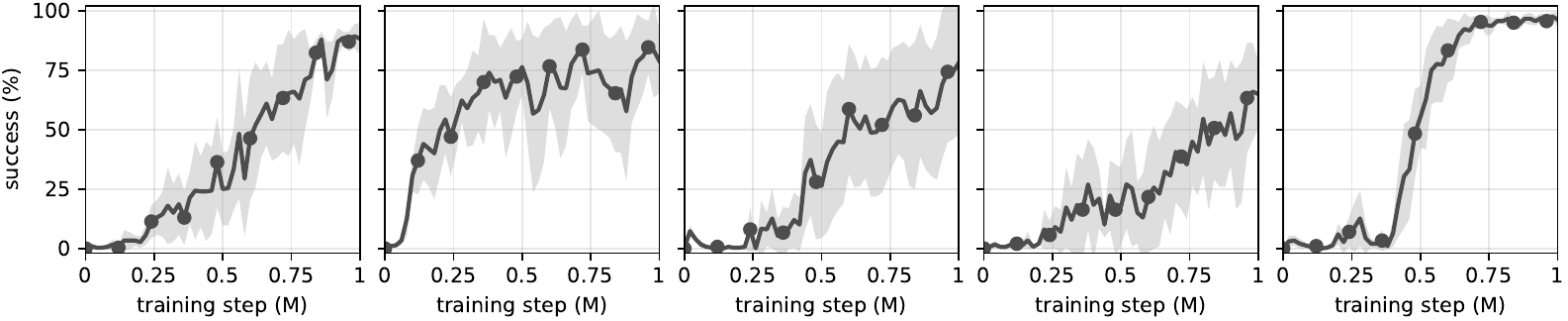} \\[6pt]
\scriptsize\texttt{humanoidmaze-large} \\
\includegraphics[width=0.95\textwidth]{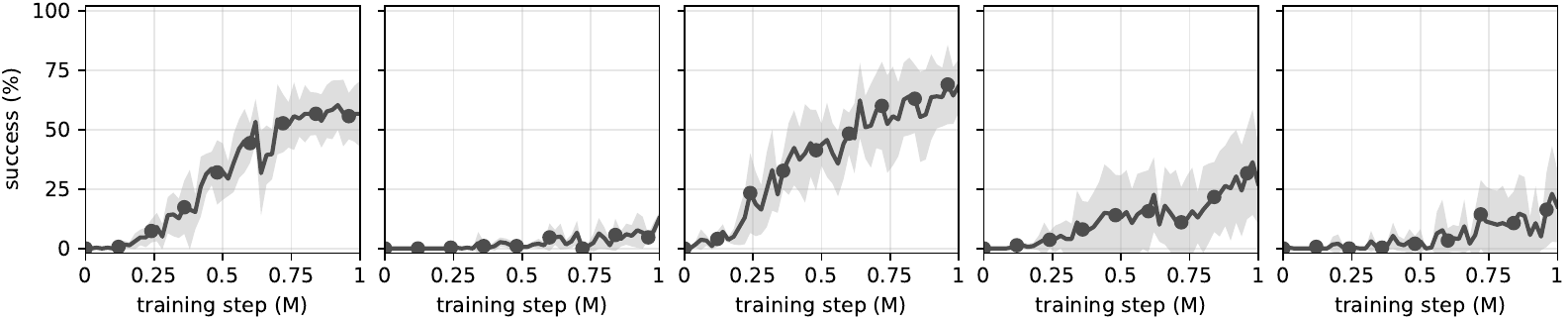} \\
\end{tabular}
\caption{VINE training curves (success rate vs.\ training steps), maze navigation domains.
Each row is one domain; the five columns are task1--task5. Curves show the mean over seeds with a 95\% confidence band.}
\label{fig:vine_training_curves_maze}
\end{figure*}

\begin{figure*}[t]
\centering
\setlength{\tabcolsep}{2pt}
\renewcommand{\arraystretch}{0.85}

\begin{tabular}{c}
\scriptsize\texttt{scene-sparse} \\
\includegraphics[width=0.95\textwidth]{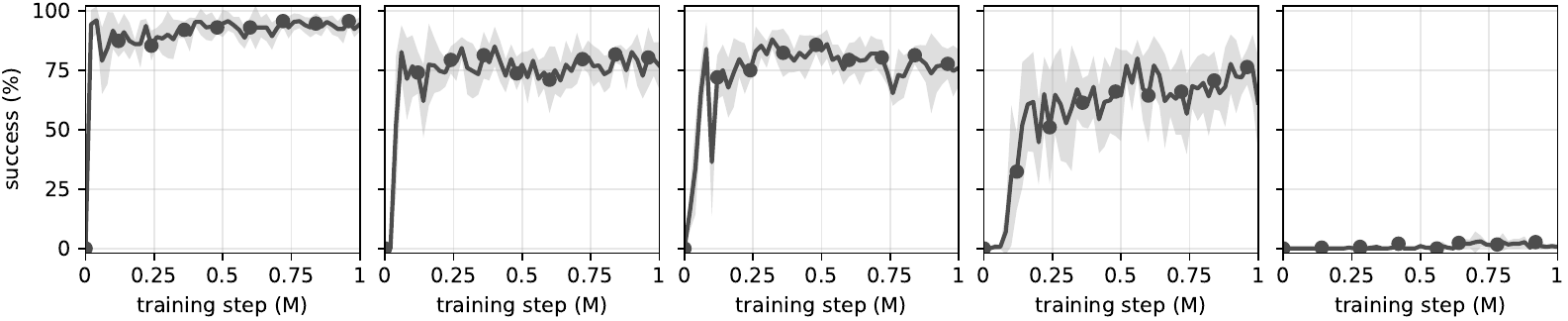} \\[6pt]
\scriptsize\texttt{puzzle-3x3-sparse} \\
\includegraphics[width=0.95\textwidth]{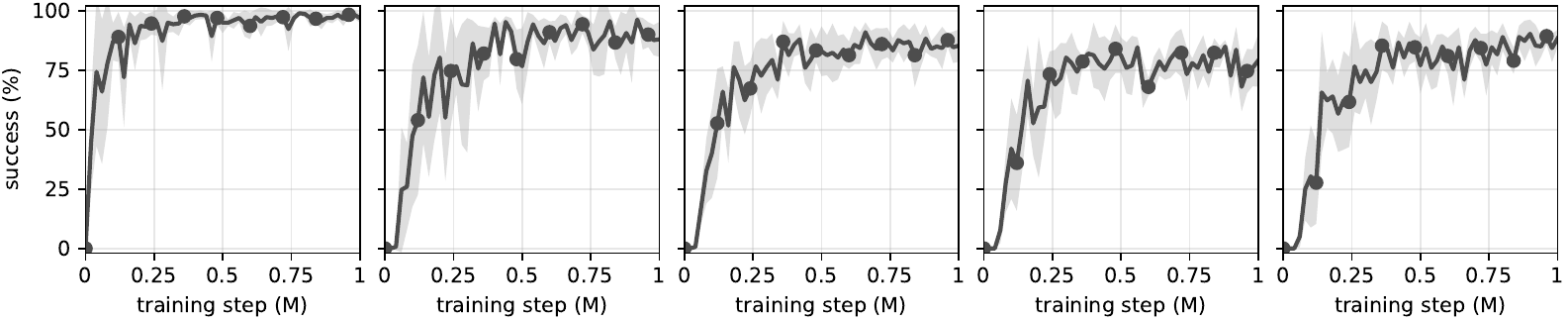} \\[6pt]
\scriptsize\texttt{cube-double} \\
\includegraphics[width=0.95\textwidth]{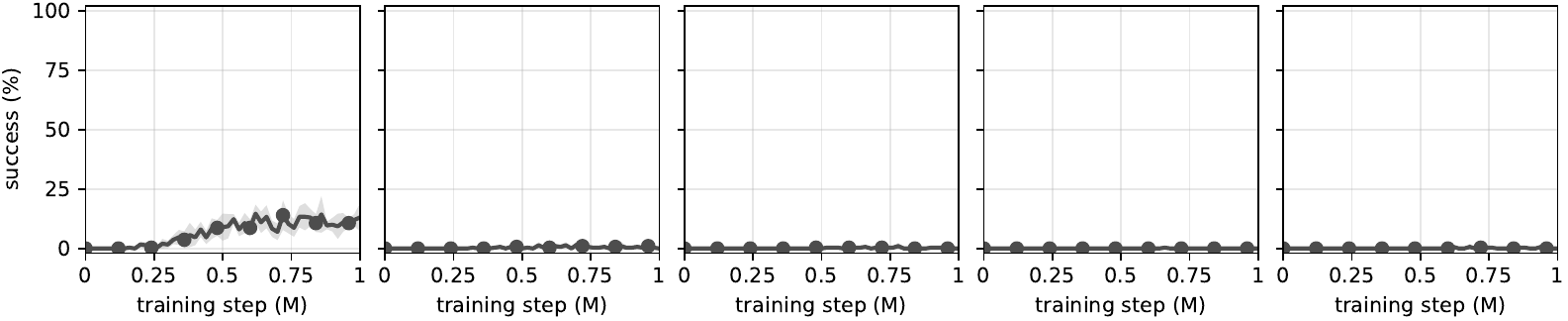} \\[6pt]
\scriptsize\texttt{cube-triple} \\
\includegraphics[width=0.95\textwidth]{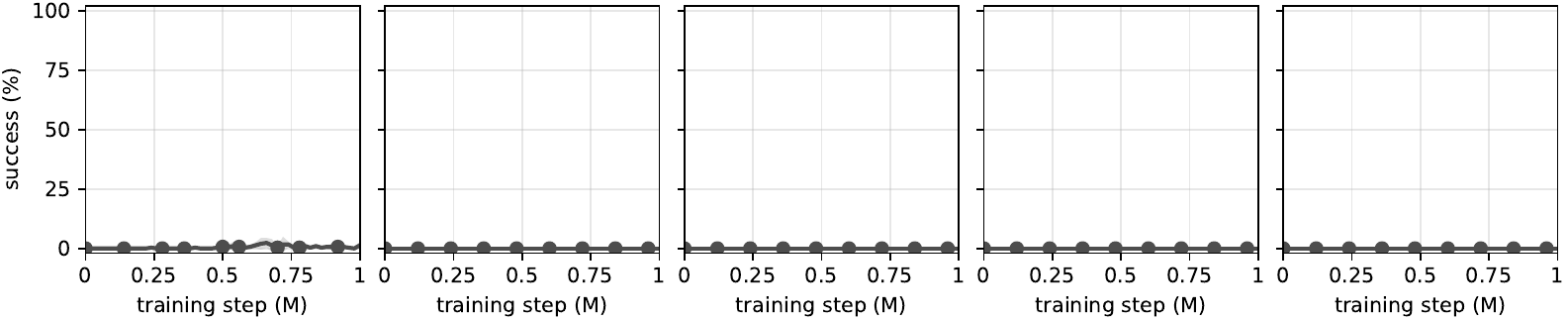} \\
\end{tabular}
\caption{VINE training curves (success rate vs.\ training steps), manipulation domains.
Each row is one domain; the five columns are task1--task5. Curves show the mean over seeds with a 95\% confidence band.}
\label{fig:vine_training_curves_manip}
\end{figure*}

\begin{table}[t]
    \centering
    \makebox[\textwidth]{\scalebox{0.58}{
    \begin{tabular}{clccccccccccccccc}
\toprule
 &  & \texttt{ReBRAC} & \texttt{FBRAC} & \texttt{BAM} & \texttt{FQL} & \texttt{FAWAC} & \texttt{CGQL} & \texttt{CGQL-M} & \texttt{CGQL-L} & \texttt{DAC} & \texttt{QSM} & \texttt{DSRL} & \texttt{FEdit} & \texttt{IFQL} & \texttt{QAM} & \texttt{VINE} \\
\midrule
 &  & $\overset{\scaleto{[97,99]}{3pt}}{\mathbf{98}}$ & $\overset{\scaleto{[0,0]}{3pt}}{\mathbf{0}}$ & $\overset{\scaleto{[88,93]}{3pt}}{\mathbf{90}}$ & $\overset{\scaleto{[89,96]}{3pt}}{\mathbf{93}}$ & $\overset{\scaleto{[3,9]}{3pt}}{\mathbf{6}}$ & $\overset{\scaleto{[49,75]}{3pt}}{\mathbf{63}}$ & $\overset{\scaleto{[49,62]}{3pt}}{\mathbf{56}}$ & $\overset{\scaleto{[31,47]}{3pt}}{\mathbf{39}}$ & $\overset{\scaleto{[82,92]}{3pt}}{\mathbf{88}}$ & $\overset{\scaleto{[81,95]}{3pt}}{\mathbf{89}}$ & $\overset{\scaleto{[53,70]}{3pt}}{\mathbf{62}}$ & $\overset{\scaleto{[60,73]}{3pt}}{\mathbf{67}}$ & $\overset{\scaleto{[26,54]}{3pt}}{\mathbf{41}}$ & $\overset{\scaleto{[67,85]}{3pt}}{\mathbf{77}}$ & $\overset{\scaleto{[100,100]}{3pt}}{\mathbf{100}}$ \\
 &  & $\overset{\scaleto{[85,91]}{3pt}}{\mathbf{88}}$ & $\overset{\scaleto{[0,0]}{3pt}}{\mathbf{0}}$ & $\overset{\scaleto{[54,65]}{3pt}}{\mathbf{60}}$ & $\overset{\scaleto{[83,89]}{3pt}}{\mathbf{86}}$ & $\overset{\scaleto{[0,2]}{3pt}}{\mathbf{1}}$ & $\overset{\scaleto{[64,79]}{3pt}}{\mathbf{73}}$ & $\overset{\scaleto{[45,54]}{3pt}}{\mathbf{49}}$ & $\overset{\scaleto{[47,55]}{3pt}}{\mathbf{51}}$ & $\overset{\scaleto{[67,75]}{3pt}}{\mathbf{71}}$ & $\overset{\scaleto{[80,88]}{3pt}}{\mathbf{84}}$ & $\overset{\scaleto{[68,82]}{3pt}}{\mathbf{75}}$ & $\overset{\scaleto{[64,70]}{3pt}}{\mathbf{67}}$ & $\overset{\scaleto{[10,19]}{3pt}}{\mathbf{14}}$ & $\overset{\scaleto{[76,83]}{3pt}}{\mathbf{80}}$ & $\overset{\scaleto{[97,99]}{3pt}}{\mathbf{98}}$ \\
 &  & $\overset{\scaleto{[97,99]}{3pt}}{\mathbf{98}}$ & $\overset{\scaleto{[5,20]}{3pt}}{\mathbf{12}}$ & $\overset{\scaleto{[92,96]}{3pt}}{\mathbf{94}}$ & $\overset{\scaleto{[48,68]}{3pt}}{\mathbf{59}}$ & $\overset{\scaleto{[32,48]}{3pt}}{\mathbf{40}}$ & $\overset{\scaleto{[89,94]}{3pt}}{\mathbf{91}}$ & $\overset{\scaleto{[87,93]}{3pt}}{\mathbf{90}}$ & $\overset{\scaleto{[75,82]}{3pt}}{\mathbf{79}}$ & $\overset{\scaleto{[97,99]}{3pt}}{\mathbf{98}}$ & $\overset{\scaleto{[96,99]}{3pt}}{\mathbf{98}}$ & $\overset{\scaleto{[77,85]}{3pt}}{\mathbf{81}}$ & $\overset{\scaleto{[53,77]}{3pt}}{\mathbf{65}}$ & $\overset{\scaleto{[44,63]}{3pt}}{\mathbf{54}}$ & $\overset{\scaleto{[86,93]}{3pt}}{\mathbf{89}}$ & $\overset{\scaleto{[100,100]}{3pt}}{\mathbf{100}}$ \\
 &  & $\overset{\scaleto{[92,95]}{3pt}}{\mathbf{94}}$ & $\overset{\scaleto{[0,0]}{3pt}}{\mathbf{0}}$ & $\overset{\scaleto{[83,88]}{3pt}}{\mathbf{85}}$ & $\overset{\scaleto{[41,68]}{3pt}}{\mathbf{54}}$ & $\overset{\scaleto{[15,20]}{3pt}}{\mathbf{18}}$ & $\overset{\scaleto{[74,83]}{3pt}}{\mathbf{78}}$ & $\overset{\scaleto{[75,81]}{3pt}}{\mathbf{78}}$ & $\overset{\scaleto{[73,80]}{3pt}}{\mathbf{76}}$ & $\overset{\scaleto{[89,94]}{3pt}}{\mathbf{91}}$ & $\overset{\scaleto{[87,93]}{3pt}}{\mathbf{90}}$ & $\overset{\scaleto{[5,33]}{3pt}}{\mathbf{18}}$ & $\overset{\scaleto{[20,36]}{3pt}}{\mathbf{28}}$ & $\overset{\scaleto{[20,29]}{3pt}}{\mathbf{25}}$ & $\overset{\scaleto{[52,81]}{3pt}}{\mathbf{69}}$ & $\overset{\scaleto{[96,100]}{3pt}}{\mathbf{98}}$ \\
 &  & $\overset{\scaleto{[94,96]}{3pt}}{\mathbf{95}}$ & $\overset{\scaleto{[0,0]}{3pt}}{\mathbf{0}}$ & $\overset{\scaleto{[86,91]}{3pt}}{\mathbf{88}}$ & $\overset{\scaleto{[83,88]}{3pt}}{\mathbf{86}}$ & $\overset{\scaleto{[17,29]}{3pt}}{\mathbf{22}}$ & $\overset{\scaleto{[69,81]}{3pt}}{\mathbf{75}}$ & $\overset{\scaleto{[76,84]}{3pt}}{\mathbf{80}}$ & $\overset{\scaleto{[75,81]}{3pt}}{\mathbf{78}}$ & $\overset{\scaleto{[90,94]}{3pt}}{\mathbf{92}}$ & $\overset{\scaleto{[91,95]}{3pt}}{\mathbf{93}}$ & $\overset{\scaleto{[65,73]}{3pt}}{\mathbf{69}}$ & $\overset{\scaleto{[56,69]}{3pt}}{\mathbf{63}}$ & $\overset{\scaleto{[34,54]}{3pt}}{\mathbf{45}}$ & $\overset{\scaleto{[86,91]}{3pt}}{\mathbf{89}}$ & $\overset{\scaleto{[99,100]}{3pt}}{\mathbf{99}}$ \\
\multirow{-6}{*}{\texttt{antmaze-large}} & \texttt{agg. (5 tasks)} & $\overset{\scaleto{[94,95]}{3pt}}{\mathbf{94}}$ & $\overset{\scaleto{[1,4]}{3pt}}{\mathbf{2}}$ & $\overset{\scaleto{[82,85]}{3pt}}{\mathbf{84}}$ & $\overset{\scaleto{[72,79]}{3pt}}{\mathbf{76}}$ & $\overset{\scaleto{[15,19]}{3pt}}{\mathbf{17}}$ & $\overset{\scaleto{[73,80]}{3pt}}{\mathbf{76}}$ & $\overset{\scaleto{[68,73]}{3pt}}{\mathbf{71}}$ & $\overset{\scaleto{[62,67]}{3pt}}{\mathbf{65}}$ & $\overset{\scaleto{[86,90]}{3pt}}{\mathbf{88}}$ & $\overset{\scaleto{[89,93]}{3pt}}{\mathbf{91}}$ & $\overset{\scaleto{[56,66]}{3pt}}{\mathbf{61}}$ & $\overset{\scaleto{[53,62]}{3pt}}{\mathbf{58}}$ & $\overset{\scaleto{[32,39]}{3pt}}{\mathbf{36}}$ & $\overset{\scaleto{[78,84]}{3pt}}{\mathbf{81}}$ & $\overset{\scaleto{[98,100]}{3pt}}{\mathbf{99}}$ \\
\midrule
 &  & $\overset{\scaleto{[31,41]}{3pt}}{\mathbf{36}}$ & $\overset{\scaleto{[0,0]}{3pt}}{\mathbf{0}}$ & $\overset{\scaleto{[0,9]}{3pt}}{\mathbf{3}}$ & $\overset{\scaleto{[0,0]}{3pt}}{\mathbf{0}}$ & $\overset{\scaleto{[0,0]}{3pt}}{\mathbf{0}}$ & $\overset{\scaleto{[0,8]}{3pt}}{\mathbf{2}}$ & $\overset{\scaleto{[0,0]}{3pt}}{\mathbf{0}}$ & $\overset{\scaleto{[0,0]}{3pt}}{\mathbf{0}}$ & $\overset{\scaleto{[45,61]}{3pt}}{\mathbf{53}}$ & $\overset{\scaleto{[57,82]}{3pt}}{\mathbf{72}}$ & $\overset{\scaleto{[0,0]}{3pt}}{\mathbf{0}}$ & $\overset{\scaleto{[0,0]}{3pt}}{\mathbf{0}}$ & $\overset{\scaleto{[0,0]}{3pt}}{\mathbf{0}}$ & $\overset{\scaleto{[5,9]}{3pt}}{\mathbf{7}}$ & $\overset{\scaleto{[44,71]}{3pt}}{\mathbf{60}}$ \\
 &  & $\overset{\scaleto{[68,78]}{3pt}}{\mathbf{73}}$ & $\overset{\scaleto{[0,0]}{3pt}}{\mathbf{0}}$ & $\overset{\scaleto{[0,0]}{3pt}}{\mathbf{0}}$ & $\overset{\scaleto{[0,0]}{3pt}}{\mathbf{0}}$ & $\overset{\scaleto{[0,0]}{3pt}}{\mathbf{0}}$ & $\overset{\scaleto{[0,0]}{3pt}}{\mathbf{0}}$ & $\overset{\scaleto{[0,0]}{3pt}}{\mathbf{0}}$ & $\overset{\scaleto{[2,6]}{3pt}}{\mathbf{4}}$ & $\overset{\scaleto{[0,0]}{3pt}}{\mathbf{0}}$ & $\overset{\scaleto{[0,0]}{3pt}}{\mathbf{0}}$ & $\overset{\scaleto{[0,0]}{3pt}}{\mathbf{0}}$ & $\overset{\scaleto{[0,0]}{3pt}}{\mathbf{0}}$ & $\overset{\scaleto{[0,0]}{3pt}}{\mathbf{0}}$ & $\overset{\scaleto{[0,0]}{3pt}}{\mathbf{0}}$ & $\overset{\scaleto{[83,92]}{3pt}}{\mathbf{87}}$ \\
 &  & $\overset{\scaleto{[9,18]}{3pt}}{\mathbf{13}}$ & $\overset{\scaleto{[0,0]}{3pt}}{\mathbf{0}}$ & $\overset{\scaleto{[0,0]}{3pt}}{\mathbf{0}}$ & $\overset{\scaleto{[0,0]}{3pt}}{\mathbf{0}}$ & $\overset{\scaleto{[0,0]}{3pt}}{\mathbf{0}}$ & $\overset{\scaleto{[0,0]}{3pt}}{\mathbf{0}}$ & $\overset{\scaleto{[0,0]}{3pt}}{\mathbf{0}}$ & $\overset{\scaleto{[0,0]}{3pt}}{\mathbf{0}}$ & $\overset{\scaleto{[0,0]}{3pt}}{\mathbf{0}}$ & $\overset{\scaleto{[0,0]}{3pt}}{\mathbf{0}}$ & $\overset{\scaleto{[0,0]}{3pt}}{\mathbf{0}}$ & $\overset{\scaleto{[0,0]}{3pt}}{\mathbf{0}}$ & $\overset{\scaleto{[0,2]}{3pt}}{\mathbf{1}}$ & $\overset{\scaleto{[0,1]}{3pt}}{\mathbf{0}}$ & $\overset{\scaleto{[40,62]}{3pt}}{\mathbf{51}}$ \\
 &  & $\overset{\scaleto{[59,84]}{3pt}}{\mathbf{74}}$ & $\overset{\scaleto{[0,0]}{3pt}}{\mathbf{0}}$ & $\overset{\scaleto{[0,0]}{3pt}}{\mathbf{0}}$ & $\overset{\scaleto{[0,0]}{3pt}}{\mathbf{0}}$ & $\overset{\scaleto{[0,0]}{3pt}}{\mathbf{0}}$ & $\overset{\scaleto{[0,0]}{3pt}}{\mathbf{0}}$ & $\overset{\scaleto{[0,0]}{3pt}}{\mathbf{0}}$ & $\overset{\scaleto{[0,0]}{3pt}}{\mathbf{0}}$ & $\overset{\scaleto{[0,0]}{3pt}}{\mathbf{0}}$ & $\overset{\scaleto{[0,0]}{3pt}}{\mathbf{0}}$ & $\overset{\scaleto{[6,19]}{3pt}}{\mathbf{12}}$ & $\overset{\scaleto{[3,17]}{3pt}}{\mathbf{10}}$ & $\overset{\scaleto{[1,4]}{3pt}}{\mathbf{2}}$ & $\overset{\scaleto{[25,44]}{3pt}}{\mathbf{34}}$ & $\overset{\scaleto{[81,91]}{3pt}}{\mathbf{86}}$ \\
 &  & $\overset{\scaleto{[85,92]}{3pt}}{\mathbf{89}}$ & $\overset{\scaleto{[0,0]}{3pt}}{\mathbf{0}}$ & $\overset{\scaleto{[0,3]}{3pt}}{\mathbf{1}}$ & $\overset{\scaleto{[0,0]}{3pt}}{\mathbf{0}}$ & $\overset{\scaleto{[0,0]}{3pt}}{\mathbf{0}}$ & $\overset{\scaleto{[0,0]}{3pt}}{\mathbf{0}}$ & $\overset{\scaleto{[6,40]}{3pt}}{\mathbf{22}}$ & $\overset{\scaleto{[0,26]}{3pt}}{\mathbf{11}}$ & $\overset{\scaleto{[6,47]}{3pt}}{\mathbf{25}}$ & $\overset{\scaleto{[0,4]}{3pt}}{\mathbf{1}}$ & $\overset{\scaleto{[1,3]}{3pt}}{\mathbf{2}}$ & $\overset{\scaleto{[0,0]}{3pt}}{\mathbf{0}}$ & $\overset{\scaleto{[0,5]}{3pt}}{\mathbf{2}}$ & $\overset{\scaleto{[29,68]}{3pt}}{\mathbf{49}}$ & $\overset{\scaleto{[92,97]}{3pt}}{\mathbf{95}}$ \\
\multirow{-6}{*}{\texttt{antmaze-giant}} & \texttt{agg. (5 tasks)} & $\overset{\scaleto{[53,60]}{3pt}}{\mathbf{57}}$ & $\overset{\scaleto{[0,0]}{3pt}}{\mathbf{0}}$ & $\overset{\scaleto{[0,2]}{3pt}}{\mathbf{1}}$ & $\overset{\scaleto{[0,0]}{3pt}}{\mathbf{0}}$ & $\overset{\scaleto{[0,0]}{3pt}}{\mathbf{0}}$ & $\overset{\scaleto{[0,2]}{3pt}}{\mathbf{0}}$ & $\overset{\scaleto{[1,8]}{3pt}}{\mathbf{4}}$ & $\overset{\scaleto{[1,6]}{3pt}}{\mathbf{3}}$ & $\overset{\scaleto{[11,21]}{3pt}}{\mathbf{16}}$ & $\overset{\scaleto{[11,17]}{3pt}}{\mathbf{15}}$ & $\overset{\scaleto{[1,4]}{3pt}}{\mathbf{3}}$ & $\overset{\scaleto{[1,3]}{3pt}}{\mathbf{2}}$ & $\overset{\scaleto{[0,2]}{3pt}}{\mathbf{1}}$ & $\overset{\scaleto{[14,22]}{3pt}}{\mathbf{18}}$ & $\overset{\scaleto{[68,83]}{3pt}}{\mathbf{76}}$ \\
\midrule
 &  & $\overset{\scaleto{[27,50]}{3pt}}{\mathbf{38}}$ & $\overset{\scaleto{[21,30]}{3pt}}{\mathbf{26}}$ & $\overset{\scaleto{[47,52]}{3pt}}{\mathbf{49}}$ & $\overset{\scaleto{[20,50]}{3pt}}{\mathbf{34}}$ & $\overset{\scaleto{[15,20]}{3pt}}{\mathbf{18}}$ & $\overset{\scaleto{[26,34]}{3pt}}{\mathbf{30}}$ & $\overset{\scaleto{[1,17]}{3pt}}{\mathbf{8}}$ & $\overset{\scaleto{[52,59]}{3pt}}{\mathbf{55}}$ & $\overset{\scaleto{[77,93]}{3pt}}{\mathbf{87}}$ & $\overset{\scaleto{[68,90]}{3pt}}{\mathbf{81}}$ & $\overset{\scaleto{[32,64]}{3pt}}{\mathbf{49}}$ & $\overset{\scaleto{[0,0]}{3pt}}{\mathbf{0}}$ & $\overset{\scaleto{[84,88]}{3pt}}{\mathbf{86}}$ & $\overset{\scaleto{[29,49]}{3pt}}{\mathbf{40}}$ & $\overset{\scaleto{[90,94]}{3pt}}{\mathbf{92}}$ \\
 &  & $\overset{\scaleto{[80,98]}{3pt}}{\mathbf{91}}$ & $\overset{\scaleto{[75,82]}{3pt}}{\mathbf{78}}$ & $\overset{\scaleto{[62,76]}{3pt}}{\mathbf{69}}$ & $\overset{\scaleto{[91,98]}{3pt}}{\mathbf{95}}$ & $\overset{\scaleto{[40,48]}{3pt}}{\mathbf{44}}$ & $\overset{\scaleto{[75,82]}{3pt}}{\mathbf{78}}$ & $\overset{\scaleto{[98,100]}{3pt}}{\mathbf{99}}$ & $\overset{\scaleto{[91,95]}{3pt}}{\mathbf{93}}$ & $\overset{\scaleto{[94,97]}{3pt}}{\mathbf{96}}$ & $\overset{\scaleto{[94,97]}{3pt}}{\mathbf{96}}$ & $\overset{\scaleto{[88,93]}{3pt}}{\mathbf{91}}$ & $\overset{\scaleto{[30,48]}{3pt}}{\mathbf{39}}$ & $\overset{\scaleto{[90,95]}{3pt}}{\mathbf{92}}$ & $\overset{\scaleto{[96,99]}{3pt}}{\mathbf{97}}$ & $\overset{\scaleto{[98,99]}{3pt}}{\mathbf{99}}$ \\
 &  & $\overset{\scaleto{[62,98]}{3pt}}{\mathbf{83}}$ & $\overset{\scaleto{[19,35]}{3pt}}{\mathbf{28}}$ & $\overset{\scaleto{[72,79]}{3pt}}{\mathbf{75}}$ & $\overset{\scaleto{[95,98]}{3pt}}{\mathbf{96}}$ & $\overset{\scaleto{[18,23]}{3pt}}{\mathbf{20}}$ & $\overset{\scaleto{[65,87]}{3pt}}{\mathbf{78}}$ & $\overset{\scaleto{[0,1]}{3pt}}{\mathbf{0}}$ & $\overset{\scaleto{[39,84]}{3pt}}{\mathbf{62}}$ & $\overset{\scaleto{[84,96]}{3pt}}{\mathbf{92}}$ & $\overset{\scaleto{[89,96]}{3pt}}{\mathbf{93}}$ & $\overset{\scaleto{[22,50]}{3pt}}{\mathbf{36}}$ & $\overset{\scaleto{[0,0]}{3pt}}{\mathbf{0}}$ & $\overset{\scaleto{[91,95]}{3pt}}{\mathbf{93}}$ & $\overset{\scaleto{[93,98]}{3pt}}{\mathbf{96}}$ & $\overset{\scaleto{[77,97]}{3pt}}{\mathbf{87}}$ \\
 &  & $\overset{\scaleto{[20,54]}{3pt}}{\mathbf{37}}$ & $\overset{\scaleto{[1,5]}{3pt}}{\mathbf{3}}$ & $\overset{\scaleto{[18,25]}{3pt}}{\mathbf{22}}$ & $\overset{\scaleto{[4,26]}{3pt}}{\mathbf{14}}$ & $\overset{\scaleto{[0,2]}{3pt}}{\mathbf{1}}$ & $\overset{\scaleto{[19,28]}{3pt}}{\mathbf{23}}$ & $\overset{\scaleto{[0,0]}{3pt}}{\mathbf{0}}$ & $\overset{\scaleto{[1,3]}{3pt}}{\mathbf{2}}$ & $\overset{\scaleto{[34,52]}{3pt}}{\mathbf{43}}$ & $\overset{\scaleto{[38,54]}{3pt}}{\mathbf{47}}$ & $\overset{\scaleto{[0,0]}{3pt}}{\mathbf{0}}$ & $\overset{\scaleto{[0,0]}{3pt}}{\mathbf{0}}$ & $\overset{\scaleto{[55,65]}{3pt}}{\mathbf{60}}$ & $\overset{\scaleto{[0,6]}{3pt}}{\mathbf{3}}$ & $\overset{\scaleto{[36,75]}{3pt}}{\mathbf{56}}$ \\
 &  & $\overset{\scaleto{[94,98]}{3pt}}{\mathbf{96}}$ & $\overset{\scaleto{[49,68]}{3pt}}{\mathbf{59}}$ & $\overset{\scaleto{[79,86]}{3pt}}{\mathbf{83}}$ & $\overset{\scaleto{[98,100]}{3pt}}{\mathbf{99}}$ & $\overset{\scaleto{[32,41]}{3pt}}{\mathbf{36}}$ & $\overset{\scaleto{[88,91]}{3pt}}{\mathbf{89}}$ & $\overset{\scaleto{[99,100]}{3pt}}{\mathbf{100}}$ & $\overset{\scaleto{[97,100]}{3pt}}{\mathbf{98}}$ & $\overset{\scaleto{[98,99]}{3pt}}{\mathbf{99}}$ & $\overset{\scaleto{[99,100]}{3pt}}{\mathbf{99}}$ & $\overset{\scaleto{[86,93]}{3pt}}{\mathbf{90}}$ & $\overset{\scaleto{[61,74]}{3pt}}{\mathbf{68}}$ & $\overset{\scaleto{[96,99]}{3pt}}{\mathbf{98}}$ & $\overset{\scaleto{[98,100]}{3pt}}{\mathbf{99}}$ & $\overset{\scaleto{[98,100]}{3pt}}{\mathbf{99}}$ \\
\multirow{-6}{*}{\texttt{humanoidmaze-medium}} & \texttt{agg. (5 tasks)} & $\overset{\scaleto{[65,74]}{3pt}}{\mathbf{69}}$ & $\overset{\scaleto{[37,41]}{3pt}}{\mathbf{39}}$ & $\overset{\scaleto{[58,62]}{3pt}}{\mathbf{60}}$ & $\overset{\scaleto{[63,73]}{3pt}}{\mathbf{68}}$ & $\overset{\scaleto{[22,26]}{3pt}}{\mathbf{24}}$ & $\overset{\scaleto{[57,62]}{3pt}}{\mathbf{60}}$ & $\overset{\scaleto{[40,43]}{3pt}}{\mathbf{42}}$ & $\overset{\scaleto{[57,67]}{3pt}}{\mathbf{62}}$ & $\overset{\scaleto{[81,85]}{3pt}}{\mathbf{83}}$ & $\overset{\scaleto{[80,86]}{3pt}}{\mathbf{83}}$ & $\overset{\scaleto{[48,57]}{3pt}}{\mathbf{53}}$ & $\overset{\scaleto{[20,23]}{3pt}}{\mathbf{22}}$ & $\overset{\scaleto{[85,87]}{3pt}}{\mathbf{86}}$ & $\overset{\scaleto{[64,69]}{3pt}}{\mathbf{67}}$ & $\overset{\scaleto{[79,93]}{3pt}}{\mathbf{87}}$ \\
\midrule
 &  & $\overset{\scaleto{[28,43]}{3pt}}{\mathbf{36}}$ & $\overset{\scaleto{[0,0]}{3pt}}{\mathbf{0}}$ & $\overset{\scaleto{[1,9]}{3pt}}{\mathbf{5}}$ & $\overset{\scaleto{[5,11]}{3pt}}{\mathbf{8}}$ & $\overset{\scaleto{[0,0]}{3pt}}{\mathbf{0}}$ & $\overset{\scaleto{[1,4]}{3pt}}{\mathbf{2}}$ & $\overset{\scaleto{[0,1]}{3pt}}{\mathbf{0}}$ & $\overset{\scaleto{[1,4]}{3pt}}{\mathbf{2}}$ & $\overset{\scaleto{[0,1]}{3pt}}{\mathbf{0}}$ & $\overset{\scaleto{[7,14]}{3pt}}{\mathbf{10}}$ & $\overset{\scaleto{[8,20]}{3pt}}{\mathbf{14}}$ & $\overset{\scaleto{[6,9]}{3pt}}{\mathbf{8}}$ & $\overset{\scaleto{[31,41]}{3pt}}{\mathbf{36}}$ & $\overset{\scaleto{[3,9]}{3pt}}{\mathbf{6}}$ & $\overset{\scaleto{[63,78]}{3pt}}{\mathbf{71}}$ \\
 &  & $\overset{\scaleto{[0,2]}{3pt}}{\mathbf{1}}$ & $\overset{\scaleto{[0,0]}{3pt}}{\mathbf{0}}$ & $\overset{\scaleto{[0,0]}{3pt}}{\mathbf{0}}$ & $\overset{\scaleto{[0,0]}{3pt}}{\mathbf{0}}$ & $\overset{\scaleto{[0,0]}{3pt}}{\mathbf{0}}$ & $\overset{\scaleto{[0,0]}{3pt}}{\mathbf{0}}$ & $\overset{\scaleto{[0,0]}{3pt}}{\mathbf{0}}$ & $\overset{\scaleto{[0,0]}{3pt}}{\mathbf{0}}$ & $\overset{\scaleto{[0,0]}{3pt}}{\mathbf{0}}$ & $\overset{\scaleto{[0,0]}{3pt}}{\mathbf{0}}$ & $\overset{\scaleto{[0,0]}{3pt}}{\mathbf{0}}$ & $\overset{\scaleto{[0,0]}{3pt}}{\mathbf{0}}$ & $\overset{\scaleto{[0,0]}{3pt}}{\mathbf{0}}$ & $\overset{\scaleto{[0,0]}{3pt}}{\mathbf{0}}$ & $\overset{\scaleto{[12,20]}{3pt}}{\mathbf{16}}$ \\
 &  & $\overset{\scaleto{[23,42]}{3pt}}{\mathbf{32}}$ & $\overset{\scaleto{[0,1]}{3pt}}{\mathbf{0}}$ & $\overset{\scaleto{[3,9]}{3pt}}{\mathbf{6}}$ & $\overset{\scaleto{[15,23]}{3pt}}{\mathbf{19}}$ & $\overset{\scaleto{[0,1]}{3pt}}{\mathbf{1}}$ & $\overset{\scaleto{[9,12]}{3pt}}{\mathbf{11}}$ & $\overset{\scaleto{[2,6]}{3pt}}{\mathbf{4}}$ & $\overset{\scaleto{[14,21]}{3pt}}{\mathbf{18}}$ & $\overset{\scaleto{[0,0]}{3pt}}{\mathbf{0}}$ & $\overset{\scaleto{[19,29]}{3pt}}{\mathbf{24}}$ & $\overset{\scaleto{[0,2]}{3pt}}{\mathbf{1}}$ & $\overset{\scaleto{[2,5]}{3pt}}{\mathbf{3}}$ & $\overset{\scaleto{[50,60]}{3pt}}{\mathbf{55}}$ & $\overset{\scaleto{[9,24]}{3pt}}{\mathbf{16}}$ & $\overset{\scaleto{[69,85]}{3pt}}{\mathbf{77}}$ \\
 &  & $\overset{\scaleto{[4,16]}{3pt}}{\mathbf{10}}$ & $\overset{\scaleto{[0,0]}{3pt}}{\mathbf{0}}$ & $\overset{\scaleto{[2,10]}{3pt}}{\mathbf{6}}$ & $\overset{\scaleto{[6,13]}{3pt}}{\mathbf{9}}$ & $\overset{\scaleto{[0,0]}{3pt}}{\mathbf{0}}$ & $\overset{\scaleto{[4,6]}{3pt}}{\mathbf{5}}$ & $\overset{\scaleto{[5,12]}{3pt}}{\mathbf{8}}$ & $\overset{\scaleto{[2,6]}{3pt}}{\mathbf{4}}$ & $\overset{\scaleto{[0,0]}{3pt}}{\mathbf{0}}$ & $\overset{\scaleto{[5,13]}{3pt}}{\mathbf{9}}$ & $\overset{\scaleto{[0,0]}{3pt}}{\mathbf{0}}$ & $\overset{\scaleto{[0,2]}{3pt}}{\mathbf{1}}$ & $\overset{\scaleto{[1,3]}{3pt}}{\mathbf{2}}$ & $\overset{\scaleto{[12,20]}{3pt}}{\mathbf{16}}$ & $\overset{\scaleto{[25,60]}{3pt}}{\mathbf{43}}$ \\
 &  & $\overset{\scaleto{[3,12]}{3pt}}{\mathbf{7}}$ & $\overset{\scaleto{[0,0]}{3pt}}{\mathbf{0}}$ & $\overset{\scaleto{[5,17]}{3pt}}{\mathbf{11}}$ & $\overset{\scaleto{[5,13]}{3pt}}{\mathbf{9}}$ & $\overset{\scaleto{[0,0]}{3pt}}{\mathbf{0}}$ & $\overset{\scaleto{[3,6]}{3pt}}{\mathbf{5}}$ & $\overset{\scaleto{[6,27]}{3pt}}{\mathbf{16}}$ & $\overset{\scaleto{[6,11]}{3pt}}{\mathbf{9}}$ & $\overset{\scaleto{[0,2]}{3pt}}{\mathbf{1}}$ & $\overset{\scaleto{[5,10]}{3pt}}{\mathbf{7}}$ & $\overset{\scaleto{[0,4]}{3pt}}{\mathbf{2}}$ & $\overset{\scaleto{[0,3]}{3pt}}{\mathbf{1}}$ & $\overset{\scaleto{[15,42]}{3pt}}{\mathbf{29}}$ & $\overset{\scaleto{[11,26]}{3pt}}{\mathbf{19}}$ & $\overset{\scaleto{[7,43]}{3pt}}{\mathbf{25}}$ \\
\multirow{-6}{*}{\texttt{humanoidmaze-large}} & \texttt{agg. (5 tasks)} & $\overset{\scaleto{[15,20]}{3pt}}{\mathbf{17}}$ & $\overset{\scaleto{[0,0]}{3pt}}{\mathbf{0}}$ & $\overset{\scaleto{[4,8]}{3pt}}{\mathbf{5}}$ & $\overset{\scaleto{[7,11]}{3pt}}{\mathbf{9}}$ & $\overset{\scaleto{[0,0]}{3pt}}{\mathbf{0}}$ & $\overset{\scaleto{[4,5]}{3pt}}{\mathbf{5}}$ & $\overset{\scaleto{[3,8]}{3pt}}{\mathbf{6}}$ & $\overset{\scaleto{[5,8]}{3pt}}{\mathbf{6}}$ & $\overset{\scaleto{[0,0]}{3pt}}{\mathbf{0}}$ & $\overset{\scaleto{[9,11]}{3pt}}{\mathbf{10}}$ & $\overset{\scaleto{[2,5]}{3pt}}{\mathbf{3}}$ & $\overset{\scaleto{[2,3]}{3pt}}{\mathbf{3}}$ & $\overset{\scaleto{[21,27]}{3pt}}{\mathbf{24}}$ & $\overset{\scaleto{[9,14]}{3pt}}{\mathbf{11}}$ & $\overset{\scaleto{[36,56]}{3pt}}{\mathbf{46}}$ \\
\midrule
 &  & $\overset{\scaleto{[96,99]}{3pt}}{\mathbf{98}}$ & $\overset{\scaleto{[55,77]}{3pt}}{\mathbf{66}}$ & $\overset{\scaleto{[100,100]}{3pt}}{\mathbf{100}}$ & $\overset{\scaleto{[99,100]}{3pt}}{\mathbf{99}}$ & $\overset{\scaleto{[57,67]}{3pt}}{\mathbf{62}}$ & $\overset{\scaleto{[73,85]}{3pt}}{\mathbf{79}}$ & $\overset{\scaleto{[100,100]}{3pt}}{\mathbf{100}}$ & $\overset{\scaleto{[100,100]}{3pt}}{\mathbf{100}}$ & $\overset{\scaleto{[99,100]}{3pt}}{\mathbf{100}}$ & $\overset{\scaleto{[100,100]}{3pt}}{\mathbf{100}}$ & $\overset{\scaleto{[100,100]}{3pt}}{\mathbf{100}}$ & $\overset{\scaleto{[92,97]}{3pt}}{\mathbf{95}}$ & $\overset{\scaleto{[91,95]}{3pt}}{\mathbf{93}}$ & $\overset{\scaleto{[100,100]}{3pt}}{\mathbf{100}}$ & $\overset{\scaleto{[97,99]}{3pt}}{\mathbf{98}}$ \\
 &  & $\overset{\scaleto{[87,93]}{3pt}}{\mathbf{90}}$ & $\overset{\scaleto{[73,87]}{3pt}}{\mathbf{80}}$ & $\overset{\scaleto{[98,100]}{3pt}}{\mathbf{99}}$ & $\overset{\scaleto{[65,76]}{3pt}}{\mathbf{71}}$ & $\overset{\scaleto{[11,18]}{3pt}}{\mathbf{15}}$ & $\overset{\scaleto{[79,94]}{3pt}}{\mathbf{88}}$ & $\overset{\scaleto{[98,100]}{3pt}}{\mathbf{99}}$ & $\overset{\scaleto{[97,100]}{3pt}}{\mathbf{98}}$ & $\overset{\scaleto{[98,100]}{3pt}}{\mathbf{99}}$ & $\overset{\scaleto{[76,87]}{3pt}}{\mathbf{82}}$ & $\overset{\scaleto{[100,100]}{3pt}}{\mathbf{100}}$ & $\overset{\scaleto{[95,98]}{3pt}}{\mathbf{97}}$ & $\overset{\scaleto{[58,69]}{3pt}}{\mathbf{63}}$ & $\overset{\scaleto{[97,100]}{3pt}}{\mathbf{98}}$ & $\overset{\scaleto{[86,94]}{3pt}}{\mathbf{90}}$ \\
 &  & $\overset{\scaleto{[42,59]}{3pt}}{\mathbf{51}}$ & $\overset{\scaleto{[21,60]}{3pt}}{\mathbf{41}}$ & $\overset{\scaleto{[97,100]}{3pt}}{\mathbf{99}}$ & $\overset{\scaleto{[96,98]}{3pt}}{\mathbf{97}}$ & $\overset{\scaleto{[12,17]}{3pt}}{\mathbf{14}}$ & $\overset{\scaleto{[20,27]}{3pt}}{\mathbf{23}}$ & $\overset{\scaleto{[90,95]}{3pt}}{\mathbf{92}}$ & $\overset{\scaleto{[88,94]}{3pt}}{\mathbf{91}}$ & $\overset{\scaleto{[74,84]}{3pt}}{\mathbf{79}}$ & $\overset{\scaleto{[95,98]}{3pt}}{\mathbf{97}}$ & $\overset{\scaleto{[95,99]}{3pt}}{\mathbf{97}}$ & $\overset{\scaleto{[53,68]}{3pt}}{\mathbf{61}}$ & $\overset{\scaleto{[66,75]}{3pt}}{\mathbf{71}}$ & $\overset{\scaleto{[99,100]}{3pt}}{\mathbf{100}}$ & $\overset{\scaleto{[84,88]}{3pt}}{\mathbf{86}}$ \\
 &  & $\overset{\scaleto{[48,72]}{3pt}}{\mathbf{60}}$ & $\overset{\scaleto{[26,72]}{3pt}}{\mathbf{49}}$ & $\overset{\scaleto{[94,98]}{3pt}}{\mathbf{96}}$ & $\overset{\scaleto{[90,94]}{3pt}}{\mathbf{92}}$ & $\overset{\scaleto{[66,77]}{3pt}}{\mathbf{71}}$ & $\overset{\scaleto{[0,1]}{3pt}}{\mathbf{0}}$ & $\overset{\scaleto{[1,11]}{3pt}}{\mathbf{6}}$ & $\overset{\scaleto{[71,96]}{3pt}}{\mathbf{86}}$ & $\overset{\scaleto{[2,15]}{3pt}}{\mathbf{8}}$ & $\overset{\scaleto{[98,100]}{3pt}}{\mathbf{99}}$ & $\overset{\scaleto{[100,100]}{3pt}}{\mathbf{100}}$ & $\overset{\scaleto{[28,42]}{3pt}}{\mathbf{35}}$ & $\overset{\scaleto{[97,99]}{3pt}}{\mathbf{98}}$ & $\overset{\scaleto{[99,100]}{3pt}}{\mathbf{100}}$ & $\overset{\scaleto{[81,92]}{3pt}}{\mathbf{87}}$ \\
 &  & $\overset{\scaleto{[11,44]}{3pt}}{\mathbf{27}}$ & $\overset{\scaleto{[7,31]}{3pt}}{\mathbf{17}}$ & $\overset{\scaleto{[95,97]}{3pt}}{\mathbf{96}}$ & $\overset{\scaleto{[29,37]}{3pt}}{\mathbf{33}}$ & $\overset{\scaleto{[21,33]}{3pt}}{\mathbf{27}}$ & $\overset{\scaleto{[0,1]}{3pt}}{\mathbf{0}}$ & $\overset{\scaleto{[61,84]}{3pt}}{\mathbf{74}}$ & $\overset{\scaleto{[63,70]}{3pt}}{\mathbf{66}}$ & $\overset{\scaleto{[41,62]}{3pt}}{\mathbf{53}}$ & $\overset{\scaleto{[44,56]}{3pt}}{\mathbf{50}}$ & $\overset{\scaleto{[100,100]}{3pt}}{\mathbf{100}}$ & $\overset{\scaleto{[18,30]}{3pt}}{\mathbf{24}}$ & $\overset{\scaleto{[93,97]}{3pt}}{\mathbf{95}}$ & $\overset{\scaleto{[84,90]}{3pt}}{\mathbf{87}}$ & $\overset{\scaleto{[2,8]}{3pt}}{\mathbf{5}}$ \\
\multirow{-6}{*}{\texttt{scene-sparse}} & \texttt{agg. (5 tasks)} & $\overset{\scaleto{[61,69]}{3pt}}{\mathbf{65}}$ & $\overset{\scaleto{[43,57]}{3pt}}{\mathbf{50}}$ & $\overset{\scaleto{[97,99]}{3pt}}{\mathbf{98}}$ & $\overset{\scaleto{[77,80]}{3pt}}{\mathbf{78}}$ & $\overset{\scaleto{[35,41]}{3pt}}{\mathbf{38}}$ & $\overset{\scaleto{[36,40]}{3pt}}{\mathbf{38}}$ & $\overset{\scaleto{[72,76]}{3pt}}{\mathbf{74}}$ & $\overset{\scaleto{[85,91]}{3pt}}{\mathbf{88}}$ & $\overset{\scaleto{[65,70]}{3pt}}{\mathbf{68}}$ & $\overset{\scaleto{[84,87]}{3pt}}{\mathbf{86}}$ & $\overset{\scaleto{[99,100]}{3pt}}{\mathbf{99}}$ & $\overset{\scaleto{[60,65]}{3pt}}{\mathbf{62}}$ & $\overset{\scaleto{[83,85]}{3pt}}{\mathbf{84}}$ & $\overset{\scaleto{[96,98]}{3pt}}{\mathbf{97}}$ & $\overset{\scaleto{[60,85]}{3pt}}{\mathbf{73}}$ \\
\midrule
 &  & $\overset{\scaleto{[98,100]}{3pt}}{\mathbf{99}}$ & $\overset{\scaleto{[0,2]}{3pt}}{\mathbf{1}}$ & $\overset{\scaleto{[9,48]}{3pt}}{\mathbf{26}}$ & $\overset{\scaleto{[98,100]}{3pt}}{\mathbf{99}}$ & $\overset{\scaleto{[6,9]}{3pt}}{\mathbf{8}}$ & $\overset{\scaleto{[78,94]}{3pt}}{\mathbf{87}}$ & $\overset{\scaleto{[100,100]}{3pt}}{\mathbf{100}}$ & $\overset{\scaleto{[96,99]}{3pt}}{\mathbf{98}}$ & $\overset{\scaleto{[96,100]}{3pt}}{\mathbf{98}}$ & $\overset{\scaleto{[89,98]}{3pt}}{\mathbf{94}}$ & $\overset{\scaleto{[58,100]}{3pt}}{\mathbf{83}}$ & $\overset{\scaleto{[100,100]}{3pt}}{\mathbf{100}}$ & $\overset{\scaleto{[100,100]}{3pt}}{\mathbf{100}}$ & $\overset{\scaleto{[95,100]}{3pt}}{\mathbf{98}}$ & $\overset{\scaleto{[100,100]}{3pt}}{\mathbf{100}}$ \\
 &  & $\overset{\scaleto{[61,92]}{3pt}}{\mathbf{77}}$ & $\overset{\scaleto{[0,0]}{3pt}}{\mathbf{0}}$ & $\overset{\scaleto{[51,100]}{3pt}}{\mathbf{75}}$ & $\overset{\scaleto{[59,96]}{3pt}}{\mathbf{79}}$ & $\overset{\scaleto{[1,4]}{3pt}}{\mathbf{2}}$ & $\overset{\scaleto{[38,70]}{3pt}}{\mathbf{55}}$ & $\overset{\scaleto{[100,100]}{3pt}}{\mathbf{100}}$ & $\overset{\scaleto{[83,98]}{3pt}}{\mathbf{92}}$ & $\overset{\scaleto{[42,89]}{3pt}}{\mathbf{66}}$ & $\overset{\scaleto{[46,86]}{3pt}}{\mathbf{67}}$ & $\overset{\scaleto{[58,100]}{3pt}}{\mathbf{83}}$ & $\overset{\scaleto{[100,100]}{3pt}}{\mathbf{100}}$ & $\overset{\scaleto{[100,100]}{3pt}}{\mathbf{100}}$ & $\overset{\scaleto{[100,100]}{3pt}}{\mathbf{100}}$ & $\overset{\scaleto{[97,100]}{3pt}}{\mathbf{99}}$ \\
 &  & $\overset{\scaleto{[81,89]}{3pt}}{\mathbf{85}}$ & $\overset{\scaleto{[0,0]}{3pt}}{\mathbf{0}}$ & $\overset{\scaleto{[54,98]}{3pt}}{\mathbf{78}}$ & $\overset{\scaleto{[58,100]}{3pt}}{\mathbf{83}}$ & $\overset{\scaleto{[0,2]}{3pt}}{\mathbf{1}}$ & $\overset{\scaleto{[16,33]}{3pt}}{\mathbf{24}}$ & $\overset{\scaleto{[100,100]}{3pt}}{\mathbf{100}}$ & $\overset{\scaleto{[79,94]}{3pt}}{\mathbf{87}}$ & $\overset{\scaleto{[35,71]}{3pt}}{\mathbf{54}}$ & $\overset{\scaleto{[2,9]}{3pt}}{\mathbf{5}}$ & $\overset{\scaleto{[58,100]}{3pt}}{\mathbf{83}}$ & $\overset{\scaleto{[97,99]}{3pt}}{\mathbf{98}}$ & $\overset{\scaleto{[100,100]}{3pt}}{\mathbf{100}}$ & $\overset{\scaleto{[100,100]}{3pt}}{\mathbf{100}}$ & $\overset{\scaleto{[90,95]}{3pt}}{\mathbf{93}}$ \\
 &  & $\overset{\scaleto{[46,76]}{3pt}}{\mathbf{62}}$ & $\overset{\scaleto{[0,0]}{3pt}}{\mathbf{0}}$ & $\overset{\scaleto{[80,100]}{3pt}}{\mathbf{92}}$ & $\overset{\scaleto{[61,100]}{3pt}}{\mathbf{84}}$ & $\overset{\scaleto{[0,2]}{3pt}}{\mathbf{1}}$ & $\overset{\scaleto{[19,30]}{3pt}}{\mathbf{25}}$ & $\overset{\scaleto{[100,100]}{3pt}}{\mathbf{100}}$ & $\overset{\scaleto{[74,94]}{3pt}}{\mathbf{85}}$ & $\overset{\scaleto{[66,78]}{3pt}}{\mathbf{72}}$ & $\overset{\scaleto{[40,69]}{3pt}}{\mathbf{54}}$ & $\overset{\scaleto{[58,100]}{3pt}}{\mathbf{83}}$ & $\overset{\scaleto{[94,99]}{3pt}}{\mathbf{97}}$ & $\overset{\scaleto{[100,100]}{3pt}}{\mathbf{100}}$ & $\overset{\scaleto{[100,100]}{3pt}}{\mathbf{100}}$ & $\overset{\scaleto{[86,92]}{3pt}}{\mathbf{89}}$ \\
 &  & $\overset{\scaleto{[51,88]}{3pt}}{\mathbf{70}}$ & $\overset{\scaleto{[0,2]}{3pt}}{\mathbf{1}}$ & $\overset{\scaleto{[1,26]}{3pt}}{\mathbf{9}}$ & $\overset{\scaleto{[2,9]}{3pt}}{\mathbf{6}}$ & $\overset{\scaleto{[0,3]}{3pt}}{\mathbf{1}}$ & $\overset{\scaleto{[37,58]}{3pt}}{\mathbf{47}}$ & $\overset{\scaleto{[100,100]}{3pt}}{\mathbf{100}}$ & $\overset{\scaleto{[80,96]}{3pt}}{\mathbf{88}}$ & $\overset{\scaleto{[29,73]}{3pt}}{\mathbf{51}}$ & $\overset{\scaleto{[32,57]}{3pt}}{\mathbf{46}}$ & $\overset{\scaleto{[100,100]}{3pt}}{\mathbf{100}}$ & $\overset{\scaleto{[99,100]}{3pt}}{\mathbf{100}}$ & $\overset{\scaleto{[100,100]}{3pt}}{\mathbf{100}}$ & $\overset{\scaleto{[99,100]}{3pt}}{\mathbf{100}}$ & $\overset{\scaleto{[93,98]}{3pt}}{\mathbf{96}}$ \\
\multirow{-6}{*}{\texttt{puzzle-3x3-sparse}} & \texttt{agg. (5 tasks)} & $\overset{\scaleto{[73,84]}{3pt}}{\mathbf{79}}$ & $\overset{\scaleto{[0,1]}{3pt}}{\mathbf{0}}$ & $\overset{\scaleto{[48,64]}{3pt}}{\mathbf{56}}$ & $\overset{\scaleto{[60,78]}{3pt}}{\mathbf{70}}$ & $\overset{\scaleto{[2,3]}{3pt}}{\mathbf{3}}$ & $\overset{\scaleto{[39,55]}{3pt}}{\mathbf{48}}$ & $\overset{\scaleto{[100,100]}{3pt}}{\mathbf{100}}$ & $\overset{\scaleto{[83,96]}{3pt}}{\mathbf{90}}$ & $\overset{\scaleto{[62,75]}{3pt}}{\mathbf{68}}$ & $\overset{\scaleto{[49,57]}{3pt}}{\mathbf{53}}$ & $\overset{\scaleto{[82,92]}{3pt}}{\mathbf{87}}$ & $\overset{\scaleto{[98,100]}{3pt}}{\mathbf{99}}$ & $\overset{\scaleto{[100,100]}{3pt}}{\mathbf{100}}$ & $\overset{\scaleto{[99,100]}{3pt}}{\mathbf{100}}$ & $\overset{\scaleto{[93,97]}{3pt}}{\mathbf{95}}$ \\
\midrule
 &  & $\overset{\scaleto{[24,35]}{3pt}}{\mathbf{29}}$ & $\overset{\scaleto{[0,0]}{3pt}}{\mathbf{0}}$ & $\overset{\scaleto{[80,88]}{3pt}}{\mathbf{84}}$ & $\overset{\scaleto{[77,86]}{3pt}}{\mathbf{81}}$ & $\overset{\scaleto{[6,10]}{3pt}}{\mathbf{8}}$ & $\overset{\scaleto{[51,60]}{3pt}}{\mathbf{55}}$ & $\overset{\scaleto{[45,55]}{3pt}}{\mathbf{50}}$ & $\overset{\scaleto{[58,67]}{3pt}}{\mathbf{62}}$ & $\overset{\scaleto{[32,40]}{3pt}}{\mathbf{36}}$ & $\overset{\scaleto{[60,73]}{3pt}}{\mathbf{67}}$ & $\overset{\scaleto{[88,92]}{3pt}}{\mathbf{90}}$ & $\overset{\scaleto{[73,81]}{3pt}}{\mathbf{77}}$ & $\overset{\scaleto{[14,18]}{3pt}}{\mathbf{16}}$ & $\overset{\scaleto{[80,89]}{3pt}}{\mathbf{85}}$ & $\overset{\scaleto{[11,20]}{3pt}}{\mathbf{16}}$ \\
 &  & $\overset{\scaleto{[3,10]}{3pt}}{\mathbf{6}}$ & $\overset{\scaleto{[0,0]}{3pt}}{\mathbf{0}}$ & $\overset{\scaleto{[40,56]}{3pt}}{\mathbf{49}}$ & $\overset{\scaleto{[40,52]}{3pt}}{\mathbf{46}}$ & $\overset{\scaleto{[0,1]}{3pt}}{\mathbf{0}}$ & $\overset{\scaleto{[32,46]}{3pt}}{\mathbf{39}}$ & $\overset{\scaleto{[41,50]}{3pt}}{\mathbf{46}}$ & $\overset{\scaleto{[46,58]}{3pt}}{\mathbf{52}}$ & $\overset{\scaleto{[31,40]}{3pt}}{\mathbf{35}}$ & $\overset{\scaleto{[59,70]}{3pt}}{\mathbf{65}}$ & $\overset{\scaleto{[82,91]}{3pt}}{\mathbf{87}}$ & $\overset{\scaleto{[21,34]}{3pt}}{\mathbf{28}}$ & $\overset{\scaleto{[10,14]}{3pt}}{\mathbf{12}}$ & $\overset{\scaleto{[73,86]}{3pt}}{\mathbf{79}}$ & $\overset{\scaleto{[2,2]}{3pt}}{\mathbf{2}}$ \\
 &  & $\overset{\scaleto{[1,4]}{3pt}}{\mathbf{2}}$ & $\overset{\scaleto{[0,0]}{3pt}}{\mathbf{0}}$ & $\overset{\scaleto{[31,45]}{3pt}}{\mathbf{38}}$ & $\overset{\scaleto{[35,49]}{3pt}}{\mathbf{42}}$ & $\overset{\scaleto{[0,0]}{3pt}}{\mathbf{0}}$ & $\overset{\scaleto{[36,50]}{3pt}}{\mathbf{44}}$ & $\overset{\scaleto{[46,55]}{3pt}}{\mathbf{50}}$ & $\overset{\scaleto{[49,57]}{3pt}}{\mathbf{52}}$ & $\overset{\scaleto{[26,36]}{3pt}}{\mathbf{31}}$ & $\overset{\scaleto{[50,63]}{3pt}}{\mathbf{57}}$ & $\overset{\scaleto{[82,89]}{3pt}}{\mathbf{85}}$ & $\overset{\scaleto{[36,51]}{3pt}}{\mathbf{44}}$ & $\overset{\scaleto{[8,12]}{3pt}}{\mathbf{10}}$ & $\overset{\scaleto{[47,61]}{3pt}}{\mathbf{54}}$ & $\overset{\scaleto{[0,2]}{3pt}}{\mathbf{1}}$ \\
 &  & $\overset{\scaleto{[0,2]}{3pt}}{\mathbf{1}}$ & $\overset{\scaleto{[0,0]}{3pt}}{\mathbf{0}}$ & $\overset{\scaleto{[5,12]}{3pt}}{\mathbf{8}}$ & $\overset{\scaleto{[8,12]}{3pt}}{\mathbf{10}}$ & $\overset{\scaleto{[0,0]}{3pt}}{\mathbf{0}}$ & $\overset{\scaleto{[10,16]}{3pt}}{\mathbf{13}}$ & $\overset{\scaleto{[8,14]}{3pt}}{\mathbf{11}}$ & $\overset{\scaleto{[16,21]}{3pt}}{\mathbf{18}}$ & $\overset{\scaleto{[13,18]}{3pt}}{\mathbf{16}}$ & $\overset{\scaleto{[17,25]}{3pt}}{\mathbf{21}}$ & $\overset{\scaleto{[28,37]}{3pt}}{\mathbf{32}}$ & $\overset{\scaleto{[11,16]}{3pt}}{\mathbf{14}}$ & $\overset{\scaleto{[2,6]}{3pt}}{\mathbf{4}}$ & $\overset{\scaleto{[18,25]}{3pt}}{\mathbf{22}}$ & $\overset{\scaleto{[0,0]}{3pt}}{\mathbf{0}}$ \\
 &  & $\overset{\scaleto{[2,7]}{3pt}}{\mathbf{4}}$ & $\overset{\scaleto{[0,0]}{3pt}}{\mathbf{0}}$ & $\overset{\scaleto{[46,64]}{3pt}}{\mathbf{56}}$ & $\overset{\scaleto{[44,56]}{3pt}}{\mathbf{50}}$ & $\overset{\scaleto{[0,1]}{3pt}}{\mathbf{0}}$ & $\overset{\scaleto{[36,47]}{3pt}}{\mathbf{42}}$ & $\overset{\scaleto{[44,52]}{3pt}}{\mathbf{48}}$ & $\overset{\scaleto{[36,46]}{3pt}}{\mathbf{41}}$ & $\overset{\scaleto{[51,61]}{3pt}}{\mathbf{56}}$ & $\overset{\scaleto{[64,74]}{3pt}}{\mathbf{69}}$ & $\overset{\scaleto{[72,80]}{3pt}}{\mathbf{76}}$ & $\overset{\scaleto{[34,44]}{3pt}}{\mathbf{39}}$ & $\overset{\scaleto{[8,14]}{3pt}}{\mathbf{11}}$ & $\overset{\scaleto{[78,85]}{3pt}}{\mathbf{82}}$ & $\overset{\scaleto{[0,0]}{3pt}}{\mathbf{0}}$ \\
\multirow{-6}{*}{\texttt{cube-double}} & \texttt{agg. (5 tasks)} & $\overset{\scaleto{[8,10]}{3pt}}{\mathbf{9}}$ & $\overset{\scaleto{[0,0]}{3pt}}{\mathbf{0}}$ & $\overset{\scaleto{[44,50]}{3pt}}{\mathbf{47}}$ & $\overset{\scaleto{[43,49]}{3pt}}{\mathbf{46}}$ & $\overset{\scaleto{[2,2]}{3pt}}{\mathbf{2}}$ & $\overset{\scaleto{[36,41]}{3pt}}{\mathbf{38}}$ & $\overset{\scaleto{[39,43]}{3pt}}{\mathbf{41}}$ & $\overset{\scaleto{[43,47]}{3pt}}{\mathbf{45}}$ & $\overset{\scaleto{[33,36]}{3pt}}{\mathbf{35}}$ & $\overset{\scaleto{[53,58]}{3pt}}{\mathbf{56}}$ & $\overset{\scaleto{[72,76]}{3pt}}{\mathbf{74}}$ & $\overset{\scaleto{[37,43]}{3pt}}{\mathbf{40}}$ & $\overset{\scaleto{[10,12]}{3pt}}{\mathbf{11}}$ & $\overset{\scaleto{[62,66]}{3pt}}{\mathbf{64}}$ & $\overset{\scaleto{[1,6]}{3pt}}{\mathbf{4}}$ \\
\midrule
 &  & $\overset{\scaleto{[2,5]}{3pt}}{\mathbf{4}}$ & $\overset{\scaleto{[0,3]}{3pt}}{\mathbf{2}}$ & $\overset{\scaleto{[8,19]}{3pt}}{\mathbf{14}}$ & $\overset{\scaleto{[10,19]}{3pt}}{\mathbf{15}}$ & $\overset{\scaleto{[0,1]}{3pt}}{\mathbf{0}}$ & $\overset{\scaleto{[33,46]}{3pt}}{\mathbf{39}}$ & $\overset{\scaleto{[34,46]}{3pt}}{\mathbf{40}}$ & $\overset{\scaleto{[33,47]}{3pt}}{\mathbf{40}}$ & $\overset{\scaleto{[17,31]}{3pt}}{\mathbf{24}}$ & $\overset{\scaleto{[13,18]}{3pt}}{\mathbf{16}}$ & $\overset{\scaleto{[4,10]}{3pt}}{\mathbf{7}}$ & $\overset{\scaleto{[8,13]}{3pt}}{\mathbf{11}}$ & $\overset{\scaleto{[1,2]}{3pt}}{\mathbf{2}}$ & $\overset{\scaleto{[10,17]}{3pt}}{\mathbf{14}}$ & $\overset{\scaleto{[2,5]}{3pt}}{\mathbf{4}}$ \\
 &  & $\overset{\scaleto{[0,0]}{3pt}}{\mathbf{0}}$ & $\overset{\scaleto{[0,0]}{3pt}}{\mathbf{0}}$ & $\overset{\scaleto{[0,1]}{3pt}}{\mathbf{0}}$ & $\overset{\scaleto{[0,1]}{3pt}}{\mathbf{0}}$ & $\overset{\scaleto{[0,0]}{3pt}}{\mathbf{0}}$ & $\overset{\scaleto{[0,0]}{3pt}}{\mathbf{0}}$ & $\overset{\scaleto{[0,0]}{3pt}}{\mathbf{0}}$ & $\overset{\scaleto{[0,0]}{3pt}}{\mathbf{0}}$ & $\overset{\scaleto{[0,1]}{3pt}}{\mathbf{0}}$ & $\overset{\scaleto{[0,0]}{3pt}}{\mathbf{0}}$ & $\overset{\scaleto{[0,0]}{3pt}}{\mathbf{0}}$ & $\overset{\scaleto{[0,1]}{3pt}}{\mathbf{0}}$ & $\overset{\scaleto{[0,0]}{3pt}}{\mathbf{0}}$ & $\overset{\scaleto{[0,2]}{3pt}}{\mathbf{1}}$ & $\overset{\scaleto{[0,0]}{3pt}}{\mathbf{0}}$ \\
 &  & $\overset{\scaleto{[0,0]}{3pt}}{\mathbf{0}}$ & $\overset{\scaleto{[0,0]}{3pt}}{\mathbf{0}}$ & $\overset{\scaleto{[1,5]}{3pt}}{\mathbf{3}}$ & $\overset{\scaleto{[0,1]}{3pt}}{\mathbf{1}}$ & $\overset{\scaleto{[0,0]}{3pt}}{\mathbf{0}}$ & $\overset{\scaleto{[0,1]}{3pt}}{\mathbf{1}}$ & $\overset{\scaleto{[0,1]}{3pt}}{\mathbf{0}}$ & $\overset{\scaleto{[0,0]}{3pt}}{\mathbf{0}}$ & $\overset{\scaleto{[0,0]}{3pt}}{\mathbf{0}}$ & $\overset{\scaleto{[0,1]}{3pt}}{\mathbf{0}}$ & $\overset{\scaleto{[0,1]}{3pt}}{\mathbf{0}}$ & $\overset{\scaleto{[0,1]}{3pt}}{\mathbf{0}}$ & $\overset{\scaleto{[0,0]}{3pt}}{\mathbf{0}}$ & $\overset{\scaleto{[2,3]}{3pt}}{\mathbf{2}}$ & $\overset{\scaleto{[0,0]}{3pt}}{\mathbf{0}}$ \\
 &  & $\overset{\scaleto{[0,0]}{3pt}}{\mathbf{0}}$ & $\overset{\scaleto{[0,0]}{3pt}}{\mathbf{0}}$ & $\overset{\scaleto{[0,0]}{3pt}}{\mathbf{0}}$ & $\overset{\scaleto{[0,0]}{3pt}}{\mathbf{0}}$ & $\overset{\scaleto{[0,0]}{3pt}}{\mathbf{0}}$ & $\overset{\scaleto{[0,0]}{3pt}}{\mathbf{0}}$ & $\overset{\scaleto{[0,0]}{3pt}}{\mathbf{0}}$ & $\overset{\scaleto{[0,0]}{3pt}}{\mathbf{0}}$ & $\overset{\scaleto{[0,0]}{3pt}}{\mathbf{0}}$ & $\overset{\scaleto{[0,0]}{3pt}}{\mathbf{0}}$ & $\overset{\scaleto{[0,0]}{3pt}}{\mathbf{0}}$ & $\overset{\scaleto{[0,1]}{3pt}}{\mathbf{0}}$ & $\overset{\scaleto{[0,0]}{3pt}}{\mathbf{0}}$ & $\overset{\scaleto{[0,1]}{3pt}}{\mathbf{0}}$ & $\overset{\scaleto{[0,0]}{3pt}}{\mathbf{0}}$ \\
 &  & $\overset{\scaleto{[0,0]}{3pt}}{\mathbf{0}}$ & $\overset{\scaleto{[0,0]}{3pt}}{\mathbf{0}}$ & $\overset{\scaleto{[0,0]}{3pt}}{\mathbf{0}}$ & $\overset{\scaleto{[0,0]}{3pt}}{\mathbf{0}}$ & $\overset{\scaleto{[0,0]}{3pt}}{\mathbf{0}}$ & $\overset{\scaleto{[0,0]}{3pt}}{\mathbf{0}}$ & $\overset{\scaleto{[0,0]}{3pt}}{\mathbf{0}}$ & $\overset{\scaleto{[0,0]}{3pt}}{\mathbf{0}}$ & $\overset{\scaleto{[0,0]}{3pt}}{\mathbf{0}}$ & $\overset{\scaleto{[0,0]}{3pt}}{\mathbf{0}}$ & $\overset{\scaleto{[0,0]}{3pt}}{\mathbf{0}}$ & $\overset{\scaleto{[0,0]}{3pt}}{\mathbf{0}}$ & $\overset{\scaleto{[0,0]}{3pt}}{\mathbf{0}}$ & $\overset{\scaleto{[0,0]}{3pt}}{\mathbf{0}}$ & $\overset{\scaleto{[0,0]}{3pt}}{\mathbf{0}}$ \\
\multirow{-6}{*}{\texttt{cube-triple}} & \texttt{agg. (5 tasks)} & $\overset{\scaleto{[0,1]}{3pt}}{\mathbf{1}}$ & $\overset{\scaleto{[0,1]}{3pt}}{\mathbf{0}}$ & $\overset{\scaleto{[2,5]}{3pt}}{\mathbf{3}}$ & $\overset{\scaleto{[2,4]}{3pt}}{\mathbf{3}}$ & $\overset{\scaleto{[0,0]}{3pt}}{\mathbf{0}}$ & $\overset{\scaleto{[7,9]}{3pt}}{\mathbf{8}}$ & $\overset{\scaleto{[7,9]}{3pt}}{\mathbf{8}}$ & $\overset{\scaleto{[7,9]}{3pt}}{\mathbf{8}}$ & $\overset{\scaleto{[3,6]}{3pt}}{\mathbf{5}}$ & $\overset{\scaleto{[3,4]}{3pt}}{\mathbf{3}}$ & $\overset{\scaleto{[1,2]}{3pt}}{\mathbf{1}}$ & $\overset{\scaleto{[2,3]}{3pt}}{\mathbf{2}}$ & $\overset{\scaleto{[0,0]}{3pt}}{\mathbf{0}}$ & $\overset{\scaleto{[3,4]}{3pt}}{\mathbf{3}}$ & $\overset{\scaleto{[0,1]}{3pt}}{\mathbf{1}}$ \\
\midrule
\texttt{all} & \texttt{agg. (40 tasks)} & $\overset{\scaleto{[37,60]}{3pt}}{\mathbf{49}}$ & $\overset{\scaleto{[5,20]}{3pt}}{\mathbf{12}}$ & $\overset{\scaleto{[32,56]}{3pt}}{\mathbf{44}}$ & $\overset{\scaleto{[31,56]}{3pt}}{\mathbf{44}}$ & $\overset{\scaleto{[6,16]}{3pt}}{\mathbf{10}}$ & $\overset{\scaleto{[24,44]}{3pt}}{\mathbf{34}}$ & $\overset{\scaleto{[30,56]}{3pt}}{\mathbf{43}}$ & $\overset{\scaleto{[34,57]}{3pt}}{\mathbf{46}}$ & $\overset{\scaleto{[33,57]}{3pt}}{\mathbf{45}}$ & $\overset{\scaleto{[37,62]}{3pt}}{\mathbf{50}}$ & $\overset{\scaleto{[35,61]}{3pt}}{\mathbf{48}}$ & $\overset{\scaleto{[24,48]}{3pt}}{\mathbf{36}}$ & $\overset{\scaleto{[31,55]}{3pt}}{\mathbf{43}}$ & $\overset{\scaleto{[42,68]}{3pt}}{\mathbf{55}}$ & $\overset{\scaleto{[55,65]}{3pt}}{\mathbf{60}}$ \\
\bottomrule
\end{tabular}}}
    \caption{\textbf{Offline results, 8 domains.} }
    \label{tab:full-results-8dom}
\end{table}

\end{document}